\documentclass[a4paper,10pt]{article}
\PassOptionsToPackage{colorlinks=true, linkcolor=blue,citecolor=blue, urlcolor=blue}{hyperref}
\pdfoutput=1
\usepackage[utf8]{inputenc}
\usepackage{geometry}
\usepackage{graphicx}
\usepackage{amsmath}
\usepackage{amsfonts}
\usepackage{amssymb}
\usepackage{amsthm}
\usepackage{bbold}
\usepackage{subfig}
\usepackage{enumerate}
\usepackage{theoremref}

\PassOptionsToPackage{table}{xcolor}
\usepackage{tikz}
\usetikzlibrary{arrows,calc,plotmarks}
\tikzset{>=angle 60}

\newcommand{\tn}[1]{\textnormal{#1}}

\newcommand{\mc}[1]{\mathcal{#1}}

\newcommand{\R}{\mathbb{R}}

\newcommand{\T}{\top}
\newcommand{\ol}[1]{\overline{#1}}

\newcommand{\fa}{\,\forall\,}

\DeclareMathOperator{\id}{id}
\DeclareMathOperator{\argmin}{arg min}

\DeclareMathOperator{\ddiv}{div}

\DeclareMathOperator{\TV}{TV}
\DeclareMathOperator{\spt}{spt}

\DeclareMathOperator{\dens}{dens}

\theoremstyle{plain}
\newtheorem{theorem}{Theorem}[section]
\newtheorem{lemma}[theorem]{Lemma}
\newtheorem{proposition}[theorem]{Proposition}

\theoremstyle{definition}

\newtheorem{remark}[theorem]{Remark}

\newcommand{\segmeas}{\tn{SegMeas}}
\newcommand{\segcoupl}{\tn{SegCoupl}}
\newcommand{\meas}{\tn{Meas}}
\newcommand{\smeas}{S}
\newcommand{\prob}{\tn{Prob}}
\newcommand{\bcolon}{\,\colon\,}
\newcommand{\project}{\tn{Proj}}
\newcommand{\WS}{{\mc{W}_2}}
\newcommand{\WD}{W}
\newcommand{\Tan}{T}
\newcommand{\sfold}{B}

\newcommand{\emb}{\tn{Emb}}
\newcommand{\testFunctions}{{C^\infty_0}}
\newcommand{\diff}{\tn{Diff}}
\newcommand{\region}{\Omega}
\newcommand{\liftM}{F}
\newcommand{\liftMB}{{F_B}}
\newcommand{\liftC}{f}

\newcommand{\cGeo}{{c_{\tn{geo}}}}
\newcommand{\feat}{\mc{F}}
\newcommand{\XEmb}{T}
\newcommand{\XCoef}{\lambda}
\newcommand{\XMode}{t}
\newcommand{\XEmbFeat}{\hat{T}}
\newcommand{\XModeFeat}{\hat{t}}
\newcommand{\XCoefSet}{\Lambda}
\newcommand{\subdiv}{\tn{\texttt{subdiv}}}
\newcommand{\adjY}{\mc{G}}
\newcommand{\cMin}{{c_{\tn{min}}}}

\numberwithin{equation}{section}

\title{Globally Optimal Joint Image Segmentation and Shape Matching Based on Wasserstein Modes}
\author{Bernhard Schmitzer and Christoph Schn\"orr}
\date{\today}

\begin{document}
\maketitle

\begin{abstract}
A functional for joint variational object segmentation and shape matching is developed.
The formulation is based on optimal transport w.r.t.~geometric distance and local feature similarity. Geometric invariance and modelling of object-typical statistical variations is achieved by introducing degrees of freedom that describe transformations and deformations of the shape template.
The shape model is mathematically equivalent to contour-based approaches but inference can be performed without conversion between the contour and region representations, allowing combination with other convex segmentation approaches and simplifying optimization. While the overall functional is non-convex, non-convexity is confined to a low-dimensional variable.
We propose a locally optimal alternating optimization scheme and a globally optimal branch and bound scheme, based on adaptive convex relaxation. Combining both methods allows to eliminate the delicate initialization problem inherent to many contour based approaches while remaining computationally practical.

The properties of the functional, its ability to adapt to a wide range of input data structures and the different optimization schemes are illustrated and compared by numerical experiments.
\end{abstract}

\tableofcontents

\section{Introduction}
	\subsection{Motivation}
		Object segmentation and matching are fundamental problems in image processing and computer vision as they form the basis for many high-level approaches to understanding an image.
		They are intimately related: segmentation of the foreground is a prerequisite for matching in a sequential processing pipeline. Whereas, when performed simultaneously, matching with a template of the sought-after object (e.g.~starfish, car, etc.) as prior knowledge can help guiding segmentation to become more robust to corruption of local image features through noise, occlusion and other distortions.
		Naturally the combined problem is more complicated.
		
		Today convex variational methods can solve image labelling and segmentation problems based on local cues exactly or in good approximation. But combining an object segmentation functional with a shape prior entails a delicate trade-off between descriptive power and computational complexity. Sophisticated shape priors are often described by highly non-convex functionals whereas convex shape prior functionals tend to be rather simplistic.
		Incompatibility between different shape representations within one approach or the requirement of geometric invariance are common causes of difficulty.

		In this paper we present a shape prior functional for simultaneous object segmentation and matching which has been designed specifically to address the issues of representation incompatibility and geometric invariance. Using optimal transport and the differential geometric structure of the 2-Wasserstein space for regularization, one can combine appearance modelling, description of statistical shape variations and geometric invariance in a mathematically uniform way.
		The linear programming formulation of optimal transport due to Kantorovich allows for an adaptive convex relaxation which can be used for a globally optimal branch and bound scheme, thus avoiding the initialization problem from which most non-convex approaches suffer.
	\subsection{Related Literature}
		\label{sec:RelatedLiterature}
		\paragraph{Image Segmentation and Shape Priors.}
		Variational methods based on convex relaxations have been successfully applied to obtain globally optimal (approximate) solutions to the originally combinatorial image labelling or segmentation problem \cite{ContinuousGlobalBinary06,Lellmann-Schnoerr-SIIMS-11,Pock-et-al-CVPR09}.
		The segmentation is usually encoded by (relaxed) indicator functions which allow for simple and convex formulation of \emph{local} data matching terms and regularizers that encourage \emph{local} boundary regularity, such as total variation and its generalizations.
		However, introducing \emph{global} regularizers, such as shape priors, into such models is difficult. Convex shape priors based on indicator functions are conceivable but tend to be rather simplistic and lack important features such as geometric invariance \cite{CremersKlodtICCV11Moments,SchmitzerSchnoerr-SSVM2011}.
		
		Somewhat complimentary is the representation of shapes by their outline contours. Treated as infinite dimensional manifolds \cite{Michor2006,sunmensoa10,PlaneCurveGeodesicExplicit-08} such representations can be used to construct sophisticated shape modelling functionals \cite{Charpiat-Shape-Statistics-05,PR-Kernel-Shape-03}. But matching contours with local image data usually yields non-convex functionals that can only be optimized locally via gradient descent. Often one has to internally convert the contour to the region representation. Therefore such approaches require a good initialization to yield reasonable results.
		
		\paragraph{Object Registration.}	
			Independent of the segmentation problem, computing meaningful registrations between two \emph{fixed} objects (e.g.~whole images, measures, meshes\ldots) has attracted a lot of attention.
			Typical applications are shape interpolation, data interpretation and using registrations as a basis for a measure of object similarity.
			Often one requires invariance of the sought-after registration under isometric transformations of either of the two objects.
			Major approaches include the framework of diffeomorphic matching and metamorphosis \cite{GlaunesTrouveYounesCVPR2004,YounesShape2010}, methods based on physical deformation energies \cite{RumpfRegression2013,HeerenShellGeodesics2012} and the metric approach to shape matching \cite{BroBroKimSIAM06,memoli-gromov-shape-11}.
			An extension to shapes that in addition to their geometry are equipped with a `signal living on the shape' is presented in \cite{TrouveFunctionalCurrents2014}.
			
			These methods provide impressive results at the cost of non-convex functionals and high computational complexity. Na\"ive online combination with object segmentation is thus not possible. In \cite{SchmitzerSchnoerr-JMIV2012} a shape prior based on object matching has been constructed through convex relaxation of the Gromov-Wasserstein distance \cite{memoli-gromov-shape-11}.
			
		\paragraph{Optimal Transport.}
			Optimal transport is a popular tool in machine learning and image analysis.
			It provides a meaningful metric on probability measures by `lifting' a metric from the base space. Thus it is a powerful similarity measure on bag-of-feature representations and other histograms \cite{Pele2009}.
			It is also applied in geometric problems to extract an object registration from the optimal transport plan \cite{OptimalTransportWarping}. However this requires alignment of the objects beforehand. A step towards loosening this constraint is presented in \cite{Guibas-EMDTransform99} where one optimizes over a suitable class of transformations.
			The 2-Wasserstein space, induced by optimal transport, exhibits structure akin to a Riemannian manifold \cite{Ambrosio2013}. This was exploited in \cite{OptimalTransportTangent2012} for analysis of spatial variations in observed sets of measures.
			
	\subsection{Contribution and Outline}
		We present a functional for object segmentation with a shape prior.
		Motivated by the literature on object registration, we propose to base the prior on matching the foreground proposal to a template object. For this we need to be able to jointly optimize over segmentation and registration.
		Matching is done via optimal transport and based both on geometry and local appearance information.
		Foreground and template are represented as metric measure spaces \cite{memoli-gromov-shape-11} which provides ample flexibility. This encompasses a wide range of spatial data structures (pixels, super-pixels, point clouds, sparse interest points, \ldots) and local appearance features (color, patches, filter responses, \ldots). %Also, it allows direct combination with convex variational image labelling functionals.
		Inspired by \cite{OptimalTransportTangent2012} the Riemannian structure of the 2-Wasserstein space is used to model geometric transformations, object-typical deformations and changes in appearance in a uniform way. Hence, the resulting approach is invariant under translation and approximately invariant under rotation and scaling.
		
		It has recently been shown that this way of modelling transformations and deformations is equivalent to modelling based on closed contours \cite{SchmitzerSchnoerr-ShapeMeasures2013} but \emph{no conversion of shape representation} is required during \emph{inference}.
		So shape modelling and local appearance matching are performed \emph{directly in the same object representation}, allowing to combine the local appearance matching of indicator functions with the manifold based shape modelling on contours.
		Also, explicitly using the conversion during \emph{learning} greatly simplifies statistical analysis of the training data and avoids difficulties that arise in \cite{OptimalTransportTangent2012}.

		The resulting overall functional is non-convex, but non-convexity is constrained to a low-dimensional variable, making optimization less cumbersome than in typical contour-based approaches or shape matching functionals.
		Using the linear programming formulation of optimal transport due to Kantorovich, we derive an adaptive convex relaxation and construct a globally optimal branch and bound scheme thereon.
		Another option is to apply a local alternating optimization scheme. By employing both optimization techniques one after another their respective advantages (no initialization required, speed) can be combined.
		This allows to construct a `coarse' object localization method and a subsequent more precise segmentation method as different approximate optimization techniques of the very same functional instead of using two different models.
		Additionally an efficient graph-cut relaxation is discussed.
		
		\paragraph{Organization.} The paper is organized as follows:
		In Sect.~\ref{sec:MathematicalBackground} the mathematical background for the paper is introduced. We touch upon the convex variational framework for image segmentation, optimal transport and its differential geometric aspects and the description of shapes via manifolds of (parametrized) contours.
		The proposed functional is successively developed throughout Sect.~\ref{sec:OTRegularization}. We start in Sect.~\ref{sec:Setup} with a basic segmentation functional where optimal transport w.r.t.~a reference template is used as a shape prior.
		This functional has obvious limitations (e.g.~lack of geometric invariance).
		An alleviation is proposed in Sect.~\ref{sec:WassersteinModes} by introducing additional degrees of freedoms that allow transformation of the template set.
		These transformations can be used to achieve geometric invariance and to model statistical object variation, learned from training data (Sects.~\ref{sec:GeometricInvariance} and \ref{sec:StatisticalVariation}).
		In Sect.~\ref{sec:Optimization} we discuss two different approaches for optimization: locally, based on alternating descending steps and globally by branch and bound with adaptive convex relaxations (Sects.~\ref{sec:OptimizationAlternating} and \ref{sec:OptimizationBnB}).
		A relaxation that replaces optimal transport by graph cuts for reduced computational cost is derived in \ref{sec:OptimizationGraphCut}.
		Numerical experiments are presented in Sect.~\ref{sec:Experiments} to illustrate the different features of the approach and to compare the two optimization schemes.
		A brief conclusion is given at the end.
	\subsection{Notation}
		For a measure space $A$ we denote by $\meas(A)$ the set of non-negative and by $\prob(A)$ the set of probability measures on $A$. For two measure spaces $A,B$ and a measurable map $f : A \rightarrow B$ we write $f_\sharp \mu$ for the push-forward of a measure $\mu$ from $A$ to $B$ which is defined by $(f_\sharp \mu)(\sigma) = \mu\big(f^{-1}(\sigma)\big)$ for all measurable $\sigma \subset B$.
		For $A \subset \R^n$ we denote by $\mc{L}_A$ the Lebesgue measure constrained to $A$ and by $|\Omega|$ the Lebesgue volume of a measurable set $\Omega \subset \R^n$. Sometimes, by abuse of notation we use $\mc{L}$ to denote the discrete approximation of the Lebesgue measure for discretized domains.
		For a product space $A \times B$ we denote by $\project_A : A \times B \rightarrow A$ the canonical projection onto some component.
		For a differentiable manifold $M$ we write $T_x M$ for the tangent space at footpoint $x \in M$.
		%For an integer $k \in \N$ denote by $[k]$ the set $\{1,\ldots,k\}$.
\section{Mathematical Background}
	\label{sec:MathematicalBackground}
	\subsection{Convex Variational Image Segmentation}
		\label{sec:Segmentation}
		Let $Y \subset \R^2$ be the (continuous) image domain.
		The goal of object segmentation is the partition of an image into fore- and background. Such a partition can be encoded by an indicator function $u : Y \rightarrow \{0,1\}$ where $u(y) = 1$ encodes that $y \in Y$ is part of the foreground.
		A typical functional for a variational segmentation approach has the form \cite{Lellmann-Schnoerr-SIIMS-11}
		\begin{align}
			\label{eq:SegmentationIntro}
			E(u) = \int_{Y} s\big(y,u(y)\big)\,dy + R(u)\,.
		\end{align}

		The first term is referred to as \emph{data term}, the second as \emph{regularizer}.
		The data term $s\big(y,u(y)\big)$ describes how well label $u(y)$ matches pixel $y$, based on local appearance information. The regularizer $R$ introduces prior knowledge to increase robustness to noisy appearance. A common assumption is that boundaries between objects are smooth, a suitable regularizer then is the \emph{total variation}.

		To obtain feasible convex problems the constraint that $u$ must be binary is usually relaxed to the interval $[0,1]$ and the functional \eqref{eq:SegmentationIntro} is suitably extended onto non-binary functions, such that it is convex. In the case of total variation regularization such an extension may be
		\begin{align}
			\label{eq:SegmentationRelaxed}
			E(u) = \int_{Y} f(y) \cdot u(y) \, dy + \TV(u)
		\end{align}
		where the data term of \eqref{eq:SegmentationIntro} can be equivalently expressed as a linear function in $u$.
		
		Total variation is a \emph{local} regularizer in the sense that it only depends locally on the (distributional) derivative of its argument. It can thus only account for local noise, i.e. noise that is statistically independent at different points of the image.
		Although this weakness can be alleviated to some extent by employing non-local total variation \cite{gilboa:1005}, the inherent underlying assumption is often not satisfied: faulty observations caused by illumination changes or occlusion clearly have long range correlations. At the same time, in particular for the problem of object segmentation more detailed prior knowledge might be available that is not exploited by local regularizers: the shape of the sought-after object. A non-local regularizer that encourages the foreground region to have a particular shape is called a \emph{shape prior}.
		
		In this article we construct a shape prior by regularization of the foreground region with optimal transport. Hence, we interpret $u$ as the density of a measure $\nu$ w.r.t.~the Lebesgue measure $\mc{L}_Y$ on $Y$.
		The feasible set for $\nu$ will be:
		\begin{align}
			\label{eq:SegmentationMeasure}
			\segmeas(Y,M) = \left\{ \vphantom{\sum}
				\nu \in \meas(Y) \bcolon
				0 \leq \nu \leq \mc{L}_Y \wedge
				\nu(Y) = M
				\right\}
		\end{align}
		The first constraint ensures that $\nu \in \segmeas(Y,M)$ has a density which is a relaxed indicator function. The second constraint fixes the overall mass of $\nu$ to $M$. This is necessary to make it comparable by optimal transport.
		
	\subsection{Optimal Transport}
		\label{sec:OptimalTransport}
		For two spaces $X$ and $Y$, two probability measures $\mu \in \prob(X)$ and $\nu \in \prob(Y)$ and a cost function $c : X \times Y \rightarrow \R$ the optimal transport cost between $\mu$ and $\nu$ is defined by
		\begin{align}
			\label{eq:OTDefinition}
			D(c;\mu,\nu) & {} = \inf_{\pi \in \Pi(\mu,\nu)} \int_{X \times Y} c(x,y)\,d\pi(x,y)
			\intertext{where}
			\Pi(\mu,\nu) & {} = \left\{ \pi \in \prob(X \times Y) \bcolon {\project_X}_\sharp \pi = \mu 
				\wedge {\project_Y}_\sharp \pi = \nu \right\}
		\end{align}
		is referred to as the set of couplings between $\mu$ and $\nu$. It is the set of non-negative measures on $X \times Y$ with marginals $\mu$ and $\nu$ respectively.
		
		For $X=Y=\R^n$ and $c(x,y) = \|x-y\|^2$ one finds that
		\begin{align}
			\WD : \prob(\R^n)^2 \rightarrow \R, \qquad \WD(\mu,\nu) = \left(D(c;\mu,\nu)\right)^{1/2}		
		\end{align}
		is a metric on the space of probability measures on $\R^n$ with finite second order moments, called the 2-Wasserstein space of $\R^n$, here denoted by $\WS(\R^n)$.
		
		This space exhibits many interesting properties. For example, for two absolutely continuous measures $\mu, \nu \in \WS(\R^n)$ \eqref{eq:OTDefinition} has a unique minimizer $\hat{\pi}$, induced by a map $T : \R^n \rightarrow \R^n$, that takes $\mu$ onto $\nu$, i.e. $\nu = T_\sharp \mu$ and
		$\hat{\pi} = (\id,T)_\sharp \mu$ and the measure valued curve
		\begin{align}
			\label{eq:OTGeodesic}
			[0,1] \ni \lambda \mapsto \big( (1-\lambda) \id + \lambda \cdot T \big)_\sharp \mu
		\end{align}
		is a geodesic between $\mu$ and $\nu$ in $\WS(\R^n)$.
		This lead to the observation that the set of absolutely continuous measures in $\WS(\R^n)$ can informally be viewed as an infinite dimensional Riemannian manifold. The \emph{tangent space} at footpoint $\mu$ is represented by gradient fields
		\begin{align}
			\Tan_\mu \WS(\R^n) = \ol{\{ \nabla \varphi : \varphi \in \testFunctions(\R^n) \}}^{L^2(\mu)}
		\end{align}
		and the Riemannian inner product for two tangent vectors is given by the $L^2$ inner product w.r.t.~$\mu$:
		\begin{align}
			\label{eq:OTInnerProduct}
			\langle t_1, t_2 \rangle_{\mu} = \int_{\R^n} \langle t_1(x), t_2(x) \rangle_{\R^2} \,d\mu(x)
		\end{align}
		Analogous to \eqref{eq:OTGeodesic} first order variations of a measure $\mu$ along a given tangent vector $t$ are described by
		\begin{align}
			\label{eq:OTDeformation}
			\lambda \mapsto (\id + \lambda \cdot t)_\sharp \mu \, .
		\end{align}
		The Jacobian determinant of $T_\lambda = \id + \lambda \cdot t$ is
		\begin{align}
			\label{eq:OTJacobian}
			\det J_{T_\lambda} = 1 + \lambda \cdot \ddiv t + \mc{O}(\lambda^2)\,.
		\end{align}
		And by the change of variables formula the density of ${T_\lambda}_\sharp \mu$ is given by
		\begin{align}
			\label{eq:OTDensity}
			\dens\big({T_\lambda}_\sharp \mu \big)\big(T_\lambda(x)\big) =
				\dens(\mu)(x) \cdot \big(1 + \lambda \cdot \ddiv t(x)\big)^{-1} + \mc{O}(\lambda^2)\,.
		\end{align}
		
		Clearly the concept of optimal transport generalizes to non-negative measures of any (finite) mass, as long as the mass of all involved measures is fixed to be identical.
		An extensive introduction to optimal transport and the structure of Wasserstein spaces is given in \cite{Villani-OptimalTransport-09}. A nice review of the Riemannian viewpoint can be found in \cite{Ambrosio2013} and is further investigated in \cite{LottWassersteinRiemannian2008} for sufficiently regular measures.
		
		In this paper we will describe the template for our shape prior by a measure $\mu$ and model geometric and statistical variations of the shape by tangent vectors $t \in \Tan_\mu \WS(\R^2)$ and their induced first-order transformations \eqref{eq:OTDeformation}.
	
	\subsection{Contour Manifolds and Shape Measures}
		\label{sec:ShapeMeasures}
		The shape of an object can be described by parametrizing its outline contour. Let $S^1$ denote the unit circle in two dimensions. The set $\emb$ of smooth embeddings of $S^1$ into $\R^2$ can be treated as an infinite dimensional manifold. A corresponding framework is laid out in \cite{MichorGlobalAnalysis}, a short summary for shape analysis is given in \cite{Michor2006}. For various proposed metrics and implementations as shape priors see references in Sect.~\ref{sec:RelatedLiterature}. Here we give a very brief summary that aids the understanding of the paper. 

		The tangent space $\Tan_c \emb$ at a given curve $c \in \emb$ is represented by smooth vector fields on $S^1$, indicating first order deformation:
		\begin{align}
			\Tan_c \emb \simeq C^\infty(S^1,\R^2)
		\end{align}
		This linear structure is a useful basis for analysis of shapes, represented by closed simple contours, and construction of shape priors thereon (see Sect.~\ref{sec:RelatedLiterature}).
		
		Let $\diff$ denote the set of smooth automorphisms on $S^1$.
		In shape analysis one naturally wants to identify different parametrizations of the same curve. This can be achieved by resorting to the quotient manifold $\sfold = \emb / \diff$ of equivalence classes of curves, equivalence $c_1 \sim c_2$ between $ c_1, c_2 \in \emb$ given if there exists a $\varphi \in \diff$ such that $c_1 = c_2 \circ \varphi$. We write $[c]$ for the class of all curves equivalent to $c$.

		We summarize:
		\begin{align}
			\emb & \text{ : smooth embeddings } S^1 \rightarrow \R^2 \nonumber \\
			\diff & \text{ : smooth automorphisms on } S^1 \nonumber \\
			\sfold & \text{ : quotient } \emb / \diff
		\end{align}		
		
		One finds that for some $a \in \Tan_c \emb$ the component which is locally tangent to the contour corresponds to a first order change in parametrization of $c$. `Actual' changes of the shape can always be represented by scalar functions on $S^1$ that describe deformations which are locally normal to the contour:
		\begin{align}
			H_{c} \emb \simeq C^\infty(S^1,\R)
		\end{align}
		where $H$ indicates that this belongs to the \emph{horizontal} bundle on $\emb$ w.r.t.~the quotient $\sfold$. For smooth paths in $\emb$ one can always find an equivalent path such that the tangents lie in $H_c \emb$.
		While splitting off reparametrization is very elegant from a mathematical perspective, it remains a computational challenge when handling parametrized curves numerically (see for example \cite{MioShapeIJCV2007}).

		Alternatively, one can represent a shape by a probability measure with constant density support on the interior of the object. Such measures and their relation to contours have been investigated in \cite{SchmitzerSchnoerr-ShapeMeasures2013}. We will here recap the main results.
		For an embedding $c \in \emb$ denote by $\region(c)$ the region enclosed by the curve and let the map $\liftM : \emb \rightarrow \WS(\R^2)$ be given by
		\begin{align}
			\big(\liftM(c)\big)(A) = |\region(c)|^{-1} \cdot |A \cap \region(c)| \qquad \text{and} \qquad
			\int \phi\,d\liftM(c) = |\region(c)|^{-1} \int_{A \cap \region(c)} \phi\,dx
		\end{align}
		for measurable $A \subset \R^2$ and integrable functions $\phi$. The set $\smeas = \liftM(\emb)$ of measures is referred to as \emph{shape measures}.
		If $c_1 \sim c_2$ then obviously $\liftM(c_1) = \liftM(c_2)$, i.e.~different parametrizations of the same curve are mapped to the same measure.
		Thus one can define a map $\liftMB : \sfold \rightarrow \WS(\R^2)$ by $\liftMB([c]) = \liftM(c)$ for any representative $c$ of equivalence class $[c]$.
		
		Consider a smooth path $\lambda \mapsto c(\lambda)$ on $\emb$ with tangents $a(\lambda) = \frac{d}{d\lambda} c(\lambda) \in H_{c(\lambda)}\emb$. The derivative $\frac{d}{d\lambda} \liftM\big(c(\lambda)\big)$ can then be represented by a vector field $t(\lambda) \in \Tan_{\liftM(c(\lambda))} \WS(\R^2)$ in the distributional sense that for any test function $\phi \in \testFunctions(\R^2)$ one has
		\begin{align}
			\label{eq:ContoursLiftingCommutation}
			\frac{d}{d\lambda} \int \phi\,d\liftM\big(c(\lambda)\big) = \int \langle \nabla \phi,t(\lambda)\rangle_{\R^2}\,d\liftM\big(c(\lambda)\big)
			\,.
		\end{align}
		For a contour $c$ the measure tangent $t \in \Tan_{\liftM(c)} \WS(\R^2)$ at $\liftM(c)$ corresponding to a contour tangent $a \in H_c \emb$ at contour $c$ in the sense of \eqref{eq:ContoursLiftingCommutation}, one has on $\region(c)$ that $t = \nabla u$ where $u$ solves the Neumann problem
		\begin{subequations}
		\label{eq:ContoursTangentLifting}
		\begin{align}
			\Delta u & {} = C \quad \text{in} \quad \region(c), \qquad \qquad
			\frac{\partial u}{\partial n} = a \circ c^{-1} \quad \text{on} \quad \partial \region(c)
			\intertext{with $\frac{\partial}{\partial n}$ denoting the derivative in outward normal direction of the contour and}
			C & {} = |\region(c)|^{-1} \int_{\partial \region(c)} a \circ c^{-1}\,ds
		\end{align}
		is the normalized total flow of $a$ through the surface $\partial \Omega(c)$.
		\end{subequations}
		This maps $a$ to a uniquely determined $t$. We denote this map by $\liftC_c$ (depending on the basis contour $c$) and write $t=\liftC_c(a)$.
		
		Note that $t = \liftC_c(a)$ has constant divergence on $\region(c) = \spt \liftM(c)$. Hence by virtue of \eqref{eq:OTDensity} one finds to first order of $\lambda$ that $\mu(\lambda) = (\id + \lambda \cdot t)_\sharp F(c)$ has constant density on its support and is therefore itself a shape measure.

		So vector fields generated as $t = \liftC_c(a)$ can said to be tangent to the set $\smeas$ in $\WS(\R^2)$ and the former can informally be regarded as a submanifold of the latter. When equipped with the proper topology it becomes a manifold in the sense of \cite{MichorGlobalAnalysis} which is diffeomorphic to $\sfold$.
		
		This means that describing shapes via shape measures and appropriate tangent vectors thereon is mathematically equivalent to describing shapes by contours modulo parametrization and deformations.
		Thus we can \emph{construct shape priors for regularization with optimal transport}, based on measures, \emph{without any representation conversion during inference and without having to handle parametrization ambiguities numerically}.

\section{Regularization with Optimal Transport}
	\label{sec:OTRegularization}
	\subsection{Setup and Basic Functional}
		\label{sec:Setup}
		Let $Y \subset \R^2$ describe the image domain in which we want to locate and match the sought-after object. As discussed in Sect.~\ref{sec:Segmentation} we will describe the object location by a relaxed indicator function $u : Y \rightarrow [0,1]$. Since we want to use optimal transport for regularization, $u$ will be interpreted as density of a measure $\nu$.
		The feasible set for $\nu$ is given by $\segmeas(Y,M)$ as defined in \eqref{eq:SegmentationMeasure} where $M$ is the total mass of the reference measure which we use for regularization.
		
		Note that this is conceptually different from matching approaches where a certain local image feature (usually intensity or gray-level) is directly converted into a density. The limitations of this are discussed in \cite{TrouveFunctionalCurrents2014} in the context of `colored currents'. In brief, one problem is, for example, that only one dimensional features can be described. Another is, that, by converting features to density, different, a priori equally important image regions, are assigned different densities and thus have a different influence on the optimizer.
		
		We use the measure to indicate the \emph{location} of the sought-after object. Local image data is handled in a \emph{different} fashion: for this we introduce a suitable \emph{feature space} $\feat$. Depending on the image this may be the corresponding color space. It may however also be a more elaborate space spanned by small image patches or local filter responses.
		We then assume that any point $y \in Y$ is equipped with some $f_y \in \feat$ which we refer to as the \emph{observed feature}. We can thus consider every pixel to be a point in the enhanced space $Y \times \feat$ with coordinates $(y,f_y)$.

		For regularization with optimal transport we need to provide a prototype, referred to as \emph{template}. Let $X$ be a set whose geometry will model the shape of the object of interest. It will be equipped with a measure $\mu$ which should usually be the Lebesgue measure on $X$, having density $1$ everywhere, to indicate that `all of $X$ is part of the object'. The constant $M$ specifying the total mass for feasible segmentations $\nu$ will be the mass of $\mu$:
		\begin{align}
			M = \mu(X)
		\end{align}
		Additionally, we describe the \emph{appearance} of the template by associating to all elements $x \in X$ corresponding $f_x \in \feat$, the \emph{expected features}.

		We assume that both the template $X$ and the image domain $Y$ are embedded into $\R^2$. The squared Euclidean distance $\|x-y\|^2$ for $x \in X$ and $y \in Y$ then provides a geometric matching cost for points:
		\begin{align}
			\cGeo(x,y) = \|x-y\|^2
		\end{align}
		Moreover, we pick some function $c_\feat : \feat \times \feat \rightarrow \R$ which models the matching cost on the feature space. Possible choices for $c_\feat$ are for example a (squared) metric, or a Bayesian log-likelihood for observing a noisy feature $f_y$ when expecting feature $f_x$.
		
		Combining this, we can construct a functional for rating the plausibility of a segmentation proposal $\nu \in \segmeas(Y,M)$:
		\begin{align}
			\label{eq:SetupBasicFunctional}
			E(\nu) = \frac{1}{2} \inf_{\pi \in \Pi(\mu,\nu)}
				\int_{X \times Y} \left( \vphantom{\sum}
					\cGeo(x,y) + c_\feat(f_x,f_y)
				\right)\,d\pi(x,y) + 
				G(\nu)
		\end{align}				
		The first term is the minimal matching cost between the segmentation region and the template via optimal transport with a cost function that combines the geometry and appearance. The second term can contain other typical components of a segmentation functional, for example a local boundary regularizer (cf.~Sect.~\ref{sec:Segmentation}). The functional is illustrated in Fig.~\ref{fig:Functionals_E0}.

		\begin{remark}[Generality of functional]
			\label{rem:Generality}
			Although we describe here a continuous setup, numerically functional \eqref{eq:SetupBasicFunctional} can be applied to a wide range of different data structures. $X$ and $Y$ can be open sets in $\R^2$, describing continuous templates and images. Then $\mu$ would be, as indicated, the Lebesgue measure on $X$ and $\mc{L}_Y$ in \eqref{eq:SegmentationMeasure} would be the Lebesgue measure on $Y$.
			Alternatively $X$ and $Y$ could be discrete sets of pixels in $\R^2$ or point clouds in $\R^n$, then $\mu$ and $\mc{L}_Y$ should be chosen to be the respective uniform counting measures on $X$ and $Y$.
			If $X$ and $Y$ represent an over-segmentation of some data (i.e. super-pixels or voxels), $\mu$ and $\mc{L}_Y$ would be weighted counting measures, the weights representing the area/volume of each cell.
		\end{remark}

		\begin{remark}[Metric structure of $\WS(\R^2)$]
			\label{rem:CFeatAndMetric}
			Adding the term $c_\feat$ to the optimal transport cost breaks the geometric structure of $\WS(\R^2)$, therefore some readers may be hesitant about this step.
			However the measure $\nu$ is an unknown variable in the approach. Therefore numerical solvers that rely on the $\WS(\R^2)$-structure cannot be applied directly, even without the $c_\feat$ term. Instead we use discrete solvers in this paper, which can simultaneously optimize for $\nu$ and $\pi$. So $c_\feat$ does not add any computational complexity whereas we gain significantly more modelling flexibility.
			Additionally, when one chooses $c_\feat$ to be a squared metric on $\feat$, then one is working on $\WS(\R^2 \times \feat)$, which also exhibits a metric structure.
		\end{remark}

		\paragraph{Limitations of the Basic Functional.}		
		Functional \eqref{eq:SetupBasicFunctional} has three major shortcomings for the application of object segmentation and shape matching, related to the choice of the embedding $X \rightarrow \R^2$:
		\begin{enumerate}[(i)]
			\item The location and orientation of the sought-after object are often unknown beforehand. Hence, a segmentation method should be invariant under Euclidean isometries, which is clearly violated by picking an arbitrary embedding $X \rightarrow \R^2$.
			If $\mu$ and  $\nu$ were fixed measures in $\meas(\R^2)$ with equal mass, then the optimal coupling for $\WD(\mu,\nu)$ would be invariant under translation (up to an adjustment of the coordinates according to the translation, of course). However, since in this application $\nu$ is not fixed this quasi-invariance cannot be exploited. Also, there is no similar invariance w.r.t.~rotation. 
			\item Any non-isometric deformation between template foreground and the object will be uniformly penalized by the geometric part of the corresponding optimal transport cost. No information on more or less common deformations (learned from a set of training samples) can be encoded.
			\item Since the mass $M$ of $\mu$, related to the size of the template $X$, equals the mass of $\nu$, this determines the size of the foreground object in $Y$. Hence, the presented functionals imply that one must know the scale of the sought-after object beforehand. This is not possible in all applications.
		\end{enumerate}
		In the next sections we will discuss how to overcome these obstacles.
		By making the embedding $X \rightarrow \R^2$ flexible, the resulting functionals become fit for (almost) isometry invariance, can handle prior information on more or less common non-isometric deformations and can dynamically adjust the object scale.

	\subsection{Wasserstein Modes}
		\label{sec:WassersteinModes}
		To overcome the limitations listed in Sect.~\ref{sec:Setup} we will allow $X$ to move and be deformed within $\R^2$. We choose the following family of embeddings:
		\begin{align}
			\label{eq:WassersteinModesTransformations}
			\XEmb_\XCoef : X \rightarrow \R^2, \qquad \XEmb_\XCoef(x) = x + \sum_{i=1}^n \XCoef_i \cdot \XMode_i(x), \qquad t_i \in T_\mu \WS(\R^2)
		\end{align}
		The transformation is parametrized by the coefficients $\XCoef \in \R^n$. This linear decomposition will allow enough flexibility for modelling while keeping the resulting functionals amenable.
		We refer to the basis maps $\{\XMode_i\}_{i=1}^n$ as \emph{modes}.
		Including the coefficients $\XCoef$ as degrees of freedom into \eqref{eq:SetupBasicFunctional} yields:
		\begin{align}
			E(\XCoef,\nu) = {} & \frac{1}{2} \inf_{\pi \in \Pi(\mu,\nu)}
				\int_{X \times Y} \left( \vphantom{\sum}
					\cGeo\big(\XEmb_\XCoef(x),y\big) + c_\feat(f_x,f_y)
				\right)\,d\pi(x,y) + {} \nonumber \\
				& \quad F(\XCoef) + G(\nu)
			\label{eq:WassersteinModesFunctional}
		\end{align}
		The function $F$ can be used to introduce statistical knowledge on the distribution of the coefficients $\XCoef$. The enhanced functional is illustrated in Fig.~\ref{fig:Functionals_E1}.
		
		\begin{figure}
			\centering
			\subfloat[]{
			\label{fig:Functionals_E0}
			\resizebox{!}{3.2cm}{
			\begin{tikzpicture}
				\begin{scope}[xshift=3cm]
					\draw[] (-0.7,-0.7) rectangle (1.2,2.7);
				\end{scope}
				\path[fill=black!10!white] (0.6,-0.3) -- (3,-0.5) -- (3.1,2.5) -- (0.6,2) -- cycle;
				\draw[double,fill=blue!30!white] (0,0) to[out=-30,in=180] (0.6,-0.3) to[out=0,in=210] (1,-0.2) to[out=30,in=-80] (1.5,1) to[out=100, in=-10] (0.6,2) to[out=170,in=60] (-0.5,1.5) to[out=240,in=150] (0,0) -- cycle;
				\begin{scope}[xshift=3cm]
					\draw[fill=red!30!white] (0,-0.5) to[out=0,in=-90] (1,1.2) to[out=90,in=-45] (0.5,1.8) to[out=135,in=0] (0.1,2.5) to[out=180,in=90] (-0.3,1.6) to[out=-90,in=90] (-0.5,0.8) to[out=-90,in=180] (0,-0.5) --cycle;
				\end{scope}
				\node at (0.5,-0.95) [anchor=center]{$X,\mu$};
				\node at (1.8,1.9) [anchor=center]{$\pi$};
				\node at (3.25,-0.95) [anchor=center]{$Y,\nu$};
			\end{tikzpicture}
			}
			}
			\hskip 0.5cm
			\subfloat[]{
			\label{fig:Functionals_E1}
			\resizebox{!}{3.2cm}{
			\begin{tikzpicture}
				\begin{scope}[xshift=6cm]
					\draw[] (-0.7,-0.7) rectangle (1.2,2.7);
				\end{scope}
				\path[fill=black!10!white] (3.7,-0.31) -- (6,-0.5) -- (6.1,2.5) -- (3.4,2.2) -- cycle;
				\draw[double,fill=blue!30!white] (0,0) to[out=-30,in=180] (0.6,-0.3) to[out=0,in=210] (1,-0.2) to[out=30,in=-80] (1.5,1) to[out=100, in=-10] (0.6,2) to[out=170,in=60] (-0.5,1.5) to[out=240,in=150] (0,0) -- cycle;
				\begin{scope}[xshift=3cm]
					\draw[fill=green!30!white] (0,0) to[out=-30,in=180] (0.7,-0.31) to[out=0,in=210] (0.8,-0.3) to[out=30,in=-80] (1.2,1.2) to[out=100, in=0] (0.4,2.2) to[out=180,in=60] (-0.3,1.6) to[out=240,in=150] (0,0) -- cycle;
				\end{scope}
				\begin{scope}[xshift=6cm]
					\draw[fill=red!30!white] (0,-0.5) to[out=0,in=-90] (1,1.2) to[out=90,in=-45] (0.5,1.8) to[out=135,in=0] (0.1,2.5) to[out=180,in=90] (-0.3,1.6) to[out=-90,in=90] (-0.5,0.8) to[out=-90,in=180] (0,-0.5) --cycle;
				\end{scope}
				\draw[->] (0.7,0.3) -- (3.7,0.2);
				\draw[->] (0.2,0.9) node [draw=black,shape=circle,fill=black,inner sep=0.5pt,label=left:$x$] {} -- (3.0,1.0) node [draw=black,shape=circle,fill=black,inner sep=0.5pt,label=right:$\XEmb_\XCoef(x)$] {};
				\draw[->] (0.55,1.4) -- (3.55,1.7);
				\node at (0.5,-0.95) [anchor=center]{$X,\mu$};
				\node at (1.925,1.9) [anchor=center]{$\XEmb_\XCoef$};
				\node at (3.35,-0.95) [anchor=center]{$\XEmb_{\XCoef\,\sharp}\,\mu$};
				\node at (4.8,1.9) [anchor=center]{$\pi$};
				\node at (6.25,-0.95) [anchor=center]{$Y,\nu$};
			\end{tikzpicture}
			}
			}
			\caption{Illustration of functionals $E(\nu)$, eq. \protect\eqref{eq:SetupBasicFunctional}, and $E(\XCoef,\nu)$, eq. \protect\eqref{eq:WassersteinModesFunctional}: %
			\protect\subref{fig:Functionals_E0} The segmentation in $Y$ is described by measure $\nu$ which is regularized by the Wasserstein distance to a template measure $\mu$, living on $X$. This simple approach introduces strong bias, depending on the relative location of $X$ and $Y$, and lacks the ability to explicitly model typical object deformations. %
			\protect\subref{fig:Functionals_E1}	In the enhanced functional the template measure $\mu$ is deformed by the map $\XEmb_\XCoef$, resulting in the push-forward $\XEmb_{\XCoef\,\sharp}\,\mu$. The segmentation $\nu$ is then regularized by its Wasserstein distance to $\XEmb_{\XCoef\,\sharp}\,\mu$. The corresponding optimal coupling $\pi$ gives a registration between the foreground part of the image and the deformed template.}
			\label{fig:Functionals}
		\end{figure}
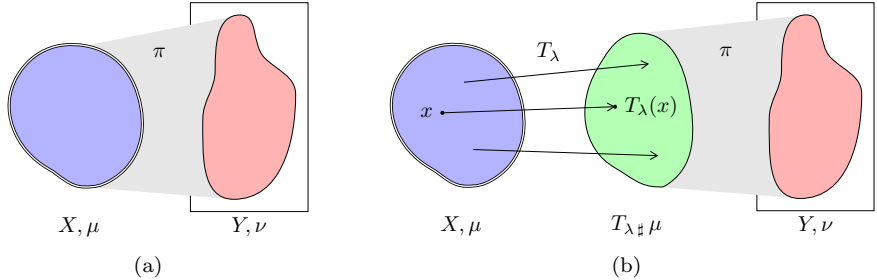
%		Throughout this paper we will assume a Gaussian distribution on $\XCoef$:
%		\begin{align}
%			F(\XCoef) = \frac{1}{2} \XCoef^\T\,\Sigma^{-1}\,\XCoef
%		\end{align}

		Functional \eqref{eq:WassersteinModesFunctional} is generally non-convex. For fixed $\XCoef$ it is convex in $\nu$.
		For fixed $\nu$ and a fixed coupling $\pi$ in the optimal transport term it is convex in $\XCoef$ if transformations are of the form \eqref{eq:WassersteinModesTransformations} and $F$ is convex.
		Joint non-convexity does not come as a surprise. It is in fact easy to see that a meaningful isometry invariant segmentation functional with explicitly modelled transformations is bound to be non-convex (Fig.~\ref{fig:Nonconvex}).
		\begin{remark}[Eliminating $\nu$]
			\label{rem:WassersteinModesBnB}
			For optimization of \eqref{eq:WassersteinModesFunctional} assume we first eliminate the high-dimensional variable $\nu$ through minimization (which is a convex problem). One is then left with:
			\begin{align}
				\label{eq:WassersteinModesXCoefFunctional}
				E_1(\XCoef) = \inf_{\nu \in \segmeas(Y,M)} E(\XCoef,\nu)
			\end{align}
			This is in general non-convex, but the dimensionality of $\XCoef$ is typically very low (of the order of 10). We can thus still hope to find globally optimal solutions by means of non-convex optimization. 
			We will present a corresponding branch and bound scheme in Sect.~\ref{sec:OptimizationBnB}.
		\end{remark}		
		\begin{remark}[Modelling transformations in feature space]
			\label{rem:cFeatTransform}
			When the feature space $\feat$ has an appropriate linear structure a natural generalization of \eqref{eq:WassersteinModesTransformations} is to not only model geometric transformations of $X$ but also of the expected features $f_x$. In analogy to \eqref{eq:WassersteinModesTransformations} consider
			\begin{align}
				\label{eq:cFeatTransformTransform}
				\XEmbFeat_\XCoef : X \rightarrow \R^2 \times \feat, \qquad \XEmbFeat_\XCoef(x) = (x,f_x) + \sum_{i=1}^n \XCoef_i \cdot \XModeFeat_i(x)
			\end{align}
			where $\XEmbFeat_0(x) = (x,f_x)$ returns the original position and expected feature of a point. The modes $\XModeFeat_i : X \rightarrow \R^2 \times \feat$ can then be used to \emph{alter both the geometry of $X$ as well as its appearance}.
			
			This will be useful when the appearance of the object is known to be subject to variations or when a feature is affected by geometric transformations: for example the expected response to an oriented local filter will need to be changed when the object is rotated.
			The corresponding \emph{generalized functional} is
			\begin{align}
				E_\feat(\XCoef,\nu) & = \frac{1}{2} \inf_{\pi \in \Pi(\mu,\nu)}
					\int_{X \times Y} \hat{c}\big(\XEmbFeat_\XCoef(x),(y,f_y)\big)\,d\pi(x,y) + 
					F(\XCoef) + G(\nu) \\
				\intertext{with}
				\label{eq:cFeatTransformCost}
				\hat{c} & : (\R^2 \times \feat)^2 \rightarrow \R, \qquad \hat{c}\big((x',f'_x),(y,f_y)\big) = \cGeo(x',y) + c_\feat(f'_x,f_y)\,.
			\end{align}
			We will further study this generalization in Sect.~\ref{sec:Experiments}. Meanwhile, for the sake of simplicity we constrain ourselves to purely geometric modes.
		\end{remark}

		\begin{figure}
			\centering
			\begin{tikzpicture}[xnode/.style={}]
				\begin{scope}[xshift=0cm,yshift=0cm]
					\path[fill=black!20!white] plot file{c4-2_f1-indicator_f2b.txt} -- cycle;
					\draw[line width=1pt,draw=red!50!black] plot file{c4-2_f1-indicator_f1.txt} -- cycle;
				\end{scope}
				\begin{scope}[xshift=4cm,yshift=0.6cm]
					\path[fill=black!20!white] plot file{c4-2_f1-indicator_f2b.txt} -- cycle;
					\draw[line width=1pt,draw=red!50!black] plot file{c4-2_f1-indicator_f1.txt} -- cycle;
				\end{scope}
				\coordinate (x1) at (0,0);
				\coordinate (x2) at (4,0.6);	
				\coordinate (origin) at (2,-1.5);
				\draw[dashed] (x1) -- (x2);
				\draw[->] (origin) -- coordinate (l1) (x1);
				\draw[->] (origin) -- coordinate (l2) (x2);
				\node at (l1) [below]{$t_1$};
				\node at (l2) [below=0.1cm]{$t_2$};
			\end{tikzpicture}
			\caption[Explicit transformation variables and non-convexity]{Explicit transformation variables and non-convexity. Gray shading indicates `foreground features'. Placing the template (red contour) at $t_1$ or $t_2$ yields equally good hypotheses. Were the prior functional convex in the translation variable, any point along the line $(1- \alpha) \cdot t_1 + \alpha \cdot t_2$ for $\alpha \in [0,1]$ would yield an at least equally good proposal, which is clearly unreasonable.}
			\label{fig:Nonconvex}
		\end{figure}
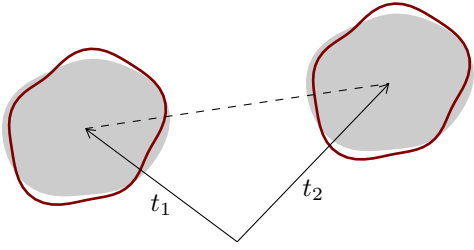

		In this paper we assume that $X = \region(c)$ for some $c \in \emb$ (cf.~Sect.~\ref{sec:ShapeMeasures}).
		As pointed out, for a meaningful template $\mu$ should be the Lebesgue measure on $X$ with constant density $1$, so $\mu$ is a (rescaled) shape measure.
		The modes $\{\XMode_i\}_i$ span a subspace of $\Tan_\mu \WS(\R^2)$ in which $\XCoef$ parametrizes a first-order deformation. We will choose $\XMode_i \in \Tan_\mu \smeas$, i.e.~tangents to the manifold of shape measures. This is equivalent to the tangent space approximation of the contour manifold $\sfold$ modulo parametrizations.
		We need to take into account how transforming $X$ through $\XEmb_\XCoef$ alters $\mu$.
		We discussed earlier that according to \eqref{eq:OTDeformation} the density of ${\XEmb_\XCoef}_\sharp \mu$ remains constant to first order.		
		Modes with non-zero divergence will lead to a density which is not $1$. Hence, ${\XEmb_\XCoef}_\sharp\,\mu$ must be rescaled accordingly, which will change its total mass and thus influence the corresponding feasible set $\segmeas(Y,M)$ for $\nu$. This will require some additional care during optimization.
		All constant-divergence modes can be decomposed into zero-divergence modes plus an additional `scale mode' (see Sect.~\ref{sec:GeometricInvariance}). We will thus see to it that all but one mode will have zero divergence and handle the scale mode with particular care (Sect.~\ref{sec:Optimization}).

	\subsection{Geometric Invariance}
		\label{sec:GeometricInvariance}
		The framework provided by transformations \eqref{eq:WassersteinModesTransformations} and functional \eqref{eq:WassersteinModesFunctional} allows to introduce geometric invariance into the segmentation\,/\,matching approach. In this section we will consider translations, (approximate) rotations and scale transformations. Scale transformations will play a special role as they change the mass of the template.
		
		The transformations will be modelled with the generators of the corresponding (local) Lie group acting on $\R^2$. Likewise invariance w.r.t.~transformation Lie groups could be introduced into matching functionals on other manifolds.
		
		\paragraph{Translation and Rotation.}
			If one chooses modes
			\begin{align}
				\XMode_{\tn{t}1}(x) = (1,0)^\T, \qquad
				\XMode_{\tn{t}2}(x) = (0,1)^\T
			\end{align}
			the corresponding coefficients $\XCoef_{\tn{t}1}, \XCoef_{\tn{t}2}$ parametrize translations of the template.
			Further, let $R(\phi)$ be the 2-dimensional rotation matrix by angle $\phi$. Then the mode
			\begin{align}
				\label{eq:RotationMode}
				\XMode_{\tn{r}}(x) = \left. \frac{d}{d\phi} R(\phi)\right|_{\phi=0}\,x = 
					(-x_2,x_1)^\T
			\end{align}
			will approximately rotate the template.\footnote{Note that $\XMode_{\tn{r}}$ is not a gradient field and thus $\notin T_\mu \WS(\R^2)$. One could find a corresponding gradient version by lifting the rotation field from the contour to the interior, Sect.~\ref{sec:ShapeMeasures}. However the functional is also meaningful with this non-gradient mode and its effect on the template is more intuitive.} This first order expansion works satisfactory for angles up to about $\pm 30^\circ$. We will consider larger rotations in the experiments, Sect.~\ref{sec:Experiments}.
			
			Note that $\XMode_{\tn{t}1},\XMode_{\tn{t}2}$ and $\XMode_{\tn{r}}$ have zero divergence. Hence, to first order the implied transformations do not alter the density of $\mu$.
			For explicit invariance under translations and rotations the modelling function $F$ in \eqref{eq:WassersteinModesFunctional} should be constant w.r.t.~the coefficients $\XCoef_{\tn{t}1}, \XCoef_{\tn{t}2}$ and $\XCoef_{\tn{r}}$.
			
		\paragraph{Scale.}
			The size of $X$ and $\mu$ determines the size of the object within the image. In many applications the scale is not known beforehand, thus dynamical resizing of the template during the search is desirable. With slight extensions the framework of transformations can be employed to introduce as a scale-mode into the approach. Let
			\begin{align}
				\XMode_{\textnormal{s}}(x) = x\,.
			\end{align}
			By the change of variable formula (cf.~(\ref{eq:OTJacobian},\ref{eq:OTDensity})) the density of ${\XEmb_\XCoef}_\sharp\,\mu$ is given by
			\begin{align}
				\dens \left( \vphantom{\sum}
					{\XEmb_\XCoef}_\sharp\,\mu
				\right)\big( \XEmb_\XCoef(x) \big) &  = \dens(\mu)(x) \cdot \big(\det J_{\XEmb_\XCoef}(x)\big)^{-1}\,.\\
				\intertext{By plugging in the scale mode $\XMode_{\textnormal{s}}$ and ignoring other modes, which due to zero divergence do not contribute to first order, we find in 2 dimensions:}
				& = (1+\XCoef_{\textnormal{s}})^{-2}
			\end{align}
			Thus, introducing a scale mode into \eqref{eq:WassersteinModesFunctional} yields
			\begin{align}
				\label{eq:GeometryScaleFunctional}
				E_{\textnormal{s}}(\XCoef,\nu) = & \frac{1}{2\,(1 + \XCoef_{\textnormal{s}})^2}
					\inf_{\pi \in \Pi\big((1+\XCoef_{\textnormal{s}})^2 \cdot \mu,\nu\big)}
					\int_{X \times Y} \left( \vphantom{\sum}
						\cGeo\big(\XEmb_\XCoef(x),y) + c_\feat(f_x,f_y)
					\right)\,d\pi(x,y)\nonumber \\
					& \qquad {} + F(\XCoef) + G(\nu)
			\end{align}
			where we have scaled $\mu$ by the appropriate factor in the feasible set for $\pi$ and we have normalized the first term by a factor of $(1 + \XCoef_{\textnormal{s}})^{-2}$ to make the term scale invariant. Depending on whether scale invariance is desired the terms $F(\XCoef)$ and $G(\nu)$ may need to be rescaled appropriately, too.
			The feasible set for $\nu$ in $E_{\textnormal{s}}$ is $\segmeas\big(Y, (1+\XCoef_{\textnormal{s}})^2 \cdot M\big)$.
			
			While the modes for translation and rotation leave the area of the template unaltered, statistical deformation modes that we learn from sample data will in general have non-zero divergence.
			Handling changes in mass will require some extra care during optimization. Therefore we will decompose such modes into a divergence-free part and a contribution of the scale-component.

	\subsection{Statistical Variation}
		\label{sec:StatisticalVariation}
		One of the limitations of \eqref{eq:SetupBasicFunctional} discussed in Sec.~\ref{sec:Setup} is that non-isometric variations of the template object are uniformly penalized by the geometric component of the corresponding optimal transport cost.
		However, not all deformations with the same optimal transport cost are equally likely. It may be necessary to reweigh the distance to more accurately model common and less common deformations.

		For contour based shape priors a model of statistical object variations is typically learned from samples in a tangent space approximation of the contour manifold.
		In \cite{OptimalTransportTangent2012} the tangent space approximation to the Wasserstein space $\WS$ was used to analyze typical deformations in a dataset of densities.
		But mimicking the learning procedure on the contour manifold with optimal transport involves some unsolved problems.

		\begin{enumerate}[(i)]
			\item The first problem is to find an appropriate footpoint for the tangent space approximation, i.e. a point by the associated tangent space of which we want to approximate the manifold to first order.
			One should pick a point which is close to all training samples. Typically one chooses a suitable mean, in a more general metric setting the natural generalization is the Karcher mean.
			Computation of the barycenter on $\WS$ is a non-trivial problem \cite{WassersteinBarycenterrohtua}, which has recently been made more accessible through Entropy smoothing \cite{Cuturi14Barycenters}. However it becomes more involved when one wants to take geometric invariances into account and impose the constraint of constant density on the support.
			In \cite{OptimalTransportTangent2012} the $L^2$-mean of the density functions was picked as footpoint after aligning the centers of mass and the principal axes of the samples.
			Though this is not necessarily an ideal choice (the $L^2$-mean of the densities can be very far from some of the samples) it seems to work for smooth densities with limited variations. It will not extend to the binary densities that we consider in this paper since their $L^2$-mean need not be binary.
			In \cite{SchmitzerSchnoerr-EMMCVPR2013} the problem was tentatively solved by manually picking a `typical' sample from the training set as the footpoint.
		
			\item The second problem is how one maps the samples into the tangent space of the footpoint. A natural choice is the logarithmic map, or some approximation thereof. %This means every sample is mapped to the tangent vector that induces the geodesic which after time $1$ reaches the sample.
			Recall from Sect.~\ref{sec:OptimalTransport} that tangent vectors on the manifold of measures are curl-free vector fields and that the logarithmic map is basically obtained by taking the relative transport map. There are some issues with the application to object segmentation:
			The vector fields computed by the logarithmic map need not have constant divergence, although fluctuations are typically small enough to be ignored for practical purposes. A second issue is that the vector fields are in general not smooth between measures with non-smooth densities, as in our case. This leads to unreasonable interpolations and unwanted artifacts during statistical analysis of the vector fields representing the sample set.
		\end{enumerate}
			
		In this paper we circumvent both problems by employing the diffeomorphism between the manifold of contours and the manifold of shape measures (see Sect.~\ref{sec:ShapeMeasures}). This allows us to outsource the shape learning problem to the contour representation where established methods for finding a good mean and tangent vectors are available (for example \cite{MioShapeIJCV2007}).
		
		Concretely we used the contour metric and the corresponding approximate algorithmic framework based on gradient descent and dynamic programming presented in \cite{MioShapeIJCV2007} for computing the Karcher mean of a set of training shapes and for mapping the training-samples onto the tangent space at the mean via the logarithmic map. We then performed a principal component analysis w.r.t.~the Riemannian inner product to extract the dominating modes of shape variation within the training set, together with their observed standard deviation $\{(t_i,\sigma_i)\}$.
		The results we obtained were stable under choosing different initializations. Learning of the class `starfish' is illustrated in Fig.~\ref{fig:ContourLearning}.
		The standard deviations $\sigma_i$ were then used to define $F(\XCoef)$ to model a Gaussian distribution on the statistical mode parameters:
		\begin{align}
			\label{eq:FDefinition}
			F(\XCoef) = \frac{\gamma}{2} \sum_{i=1}^{n_{\tn{stat}}} \left(\frac{\XCoef_i}{\sigma_i}\right)^2
		\end{align}
	where $\gamma$ is a parameter determining the weight of $F$ w.r.t.~the other functional components.
	
		\begin{figure}
			\centering
			\begin{tabular}{cc}
			\includegraphics[width=5cm]{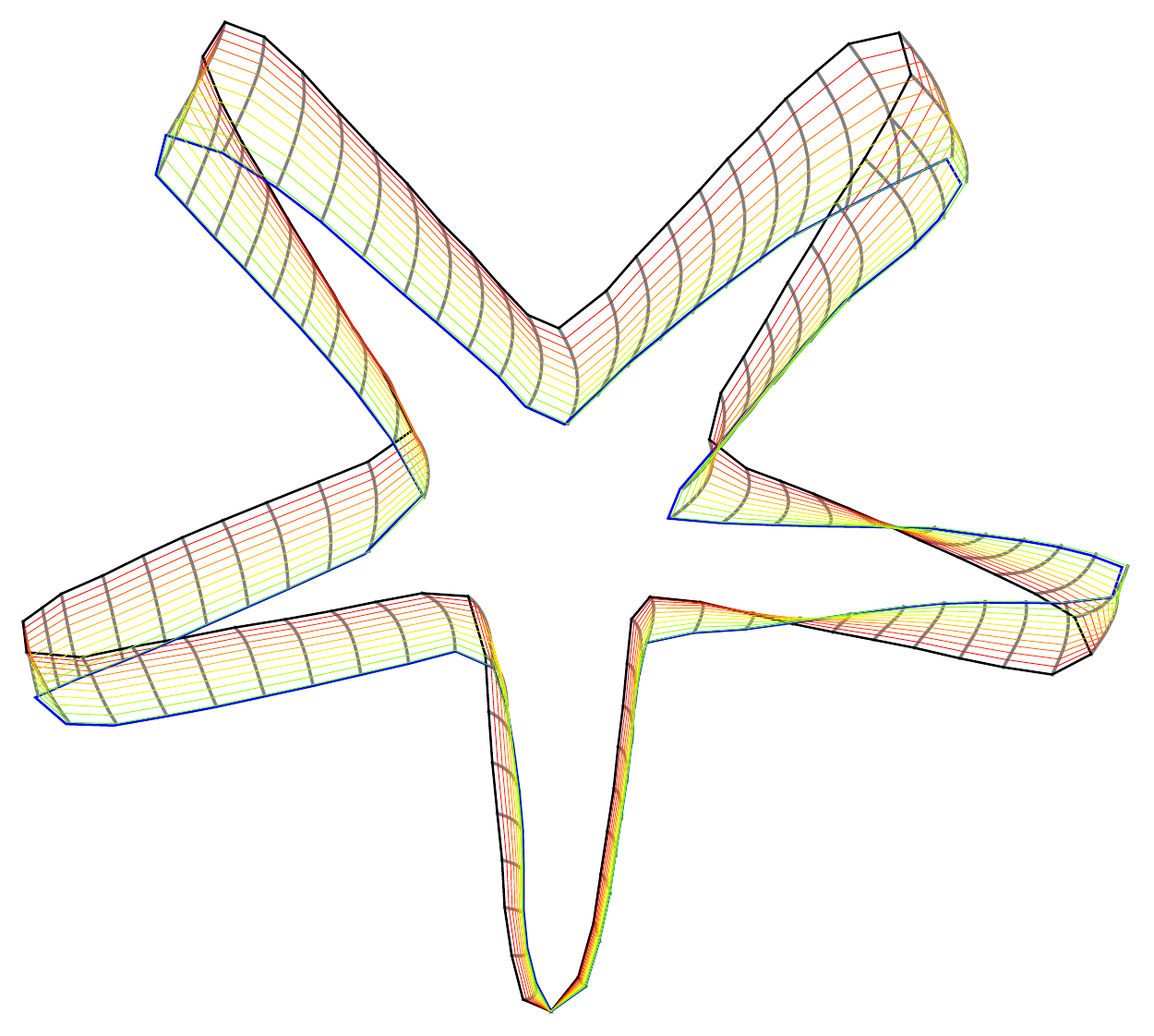} &
			\includegraphics[width=5cm]{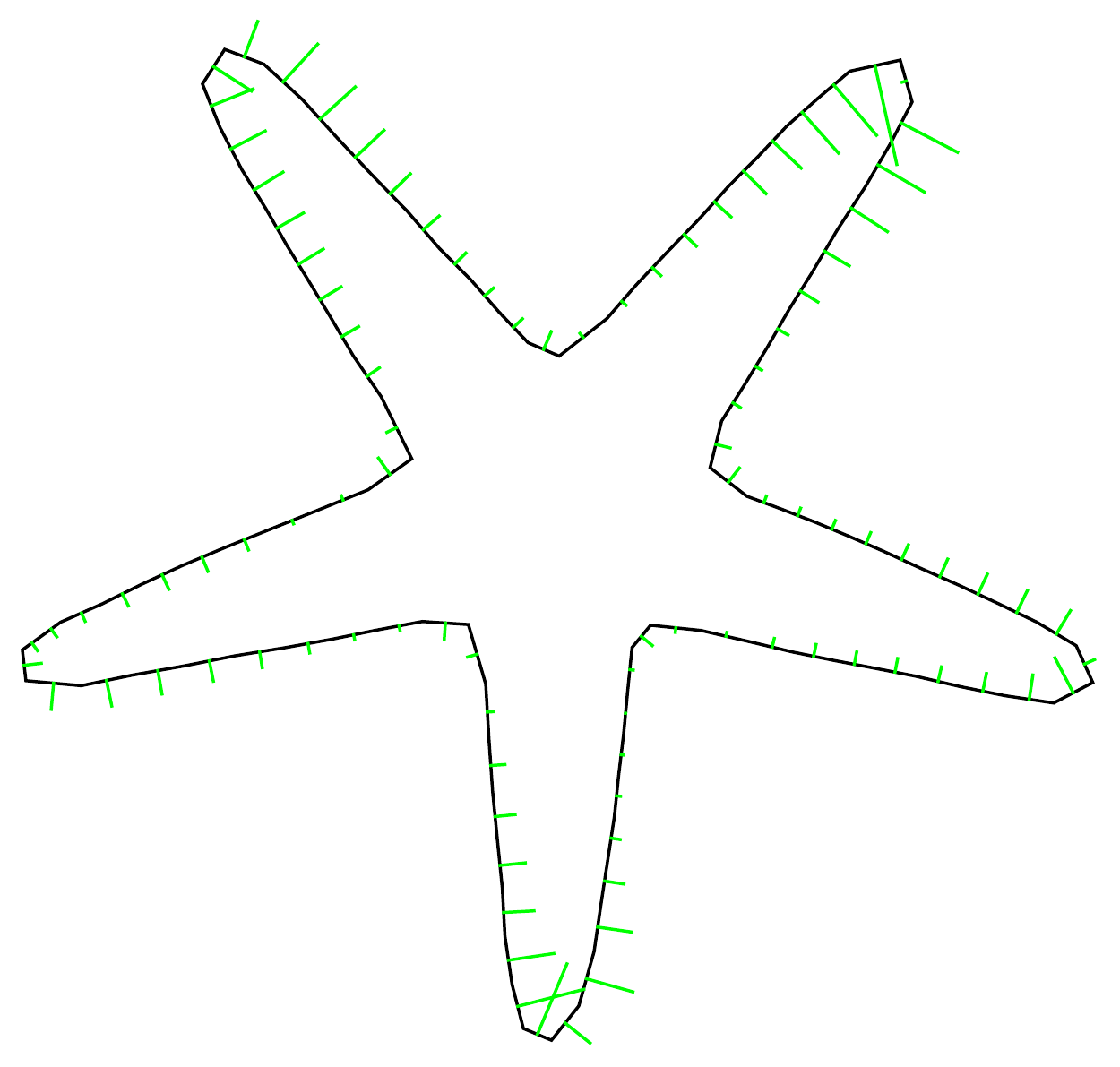} \\
			\includegraphics[width=5cm]{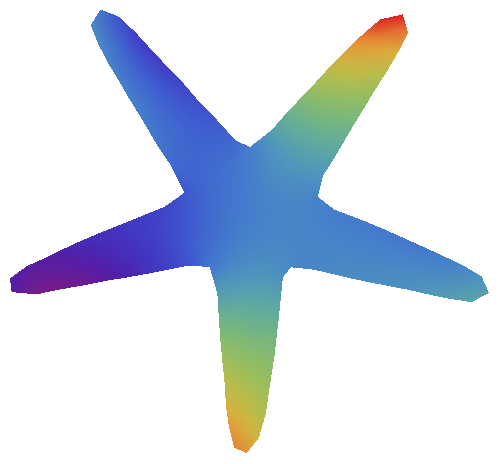} &
			\includegraphics[width=5cm]{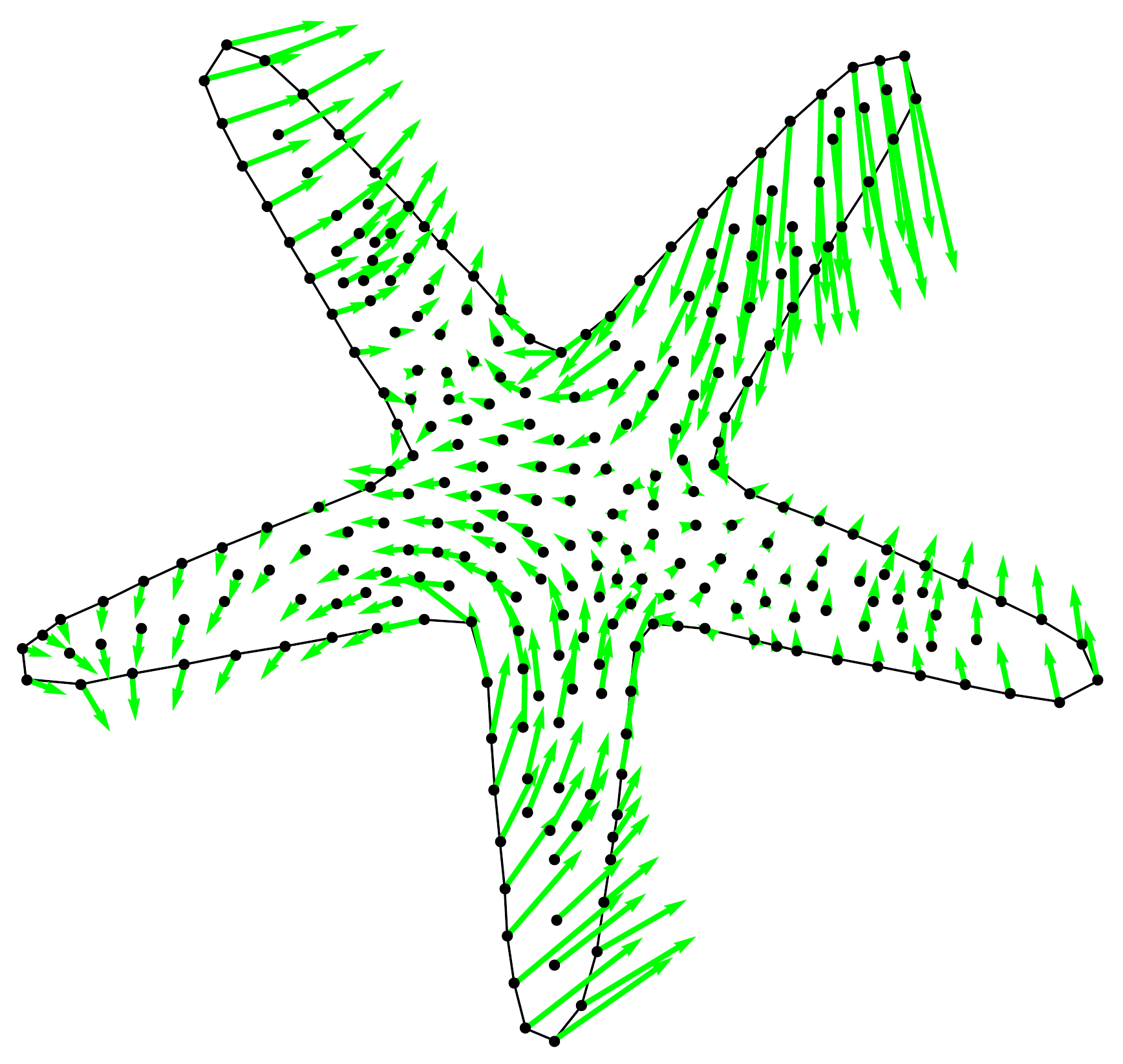} \\
			\end{tabular}
			\caption{Learning of contours. \textbf{Top left:} geodesic from shape mean to a training sample. %
			\textbf{Top right:} Normal contour deformation of first principal component of training samples. %
			\textbf{Bottom left:} Potential function $u$ for lifting the deformation to the full region (see \protect\eqref{eq:ContoursTangentLifting}). %
			\textbf{Bottom right:} Gradient field which gives deformation mode for whole template region.
			}
			\label{fig:ContourLearning}
		\end{figure}

	\subsection{Background Modelling}
		\label{sec:BackgroundModelling}
		The previous sections describe how to model the sought-after object via a template, i.e.~they focus on the image foreground. Let us now briefly comment on the background.
		
		Sometimes information on the expected appearance of the background is available. This can be incorporated by a linear contribution to $G$ \eqref{eq:WassersteinModesFunctional}:
		\begin{align}
			G(\nu) = \int_Y g(y)\,d\nu(y)
		\end{align}
		where a positive (negative) coefficient $g(y)$ indicates that a given point is likely to be part of the background (foreground) (cf.~Sect.~\ref{sec:Segmentation}). Such linear terms can be absorbed into the optimal transport term:
		\begin{align}
			\int_Y g(y)\,d\nu(y)	 = \int_{X \times Y} g(y)\,d\pi(x,y)
		\end{align}
		That is, the background appearance model leads to an effective shift of the foreground assignment costs: $c(x,y) \rightarrow c(x,y) + g(y)$.
		
		In other situations it may be desirable to impose that the region directly around the foreground object does not look like foreground itself.
		An example for such a situation and the corresponding solution are discussed with numerical examples in Sect.~\ref{sec:Experiments}, see Fig.~\ref{fig:2}.

\section{Optimization}
	\label{sec:Optimization}
	\subsection{Alternating Optimization}
		\label{sec:OptimizationAlternating}
		Functional \eqref{eq:WassersteinModesFunctional} is generally non-convex. It is convex in $\nu$ for fixed $\XCoef$ and it is convex in $\XCoef$ under suitable conditions (see Sect.~\ref{sec:WassersteinModes}).
		Based on this, an alternating optimization scheme is conceivable for divergence-free modes.
		This has also been proposed in \cite[Sect.~3.2.1]{Guibas-EMDTransform99}.
		We require the following reformulation of \eqref{eq:WassersteinModesXCoefFunctional}:

		\begin{remark}[Coupling reformulation]
			\label{rem:CouplingReformulation}
			Computing \eqref{eq:WassersteinModesXCoefFunctional} involves a nested optimization problem over $\nu \in \segmeas(Y,M)$ and then $\pi \in \Pi(\mu,\nu)$.
			Given a coupling $\pi \in \Pi(\mu,\nu)$ the marginal $\nu$ can be reconstructed via projection: $\nu = {\project_Y}_\sharp \pi$. This allows to reformulate the optimization of \eqref{eq:WassersteinModesXCoefFunctional} directly in terms of couplings. Let
			\begin{align}
				\hat{E}(\XCoef,\pi) = \frac{1}{2}
				\int_{X \times Y} \left( \vphantom{\sum}
					\cGeo\big(\XEmb_\XCoef(x),y\big) + c_\feat(f_x,f_y)
				\right)\,d\pi(x,y) + 
				F(\XCoef) + G({\project_Y}_\sharp \pi)
			\end{align}
			and let the feasible set for $\pi$ in $\hat{E}$ be
			\begin{align}
				\segcoupl(Y,\mu) = & \bigcup_{\nu \in \segmeas(Y,M)} \Pi(\mu,\nu) \nonumber \\
				= & \left\{ \pi \in \meas(X \times Y) \bcolon {\project_X}_\sharp \pi = \mu \wedge {\project_Y}_\sharp \pi \leq \mc{L}_Y \right\}\,.				
			\end{align}
			Then for fixed $\XCoef$ one has by construction
			\begin{align}
				\inf_{\nu \in \segmeas(Y,M)} E(\XCoef,\nu) = \inf_{\pi \in \segcoupl(Y,\mu)} \hat{E}(\XCoef,\pi)
			\end{align}
			and for any optimizer $\pi^\ast$ of $\hat{E}$ the marginal ${\project_Y}_\sharp \pi^\ast$ is an optimizer of $E$.
		\end{remark}
		
		Functional $\hat{E}(\XCoef,\pi)$ is separately convex in $\XCoef$ and $\pi$ for transformations of the form \eqref{eq:WassersteinModesTransformations} and convex $F$.
		For some initial $\XCoef^1$ consider the following sequence for $k=1,2,\ldots$:
		\begin{subequations}
			\label{eq:OptimizationAlternating}
			\begin{align}
				\pi^{k} & \in \argmin_{ \pi \in \segcoupl(Y,\mu)} \hat{E}(\XCoef^{k},\pi)
				\label{eq:OptimizationAlternatingNu} \\
				\XCoef^{k+1} & \in \argmin_{\XCoef \in \R^n} \hat{E}(\XCoef,\pi^{k})
				\label{eq:OptimizationAlternatingXCoef}
			\end{align}
		\end{subequations}
		\begin{proposition}
			The sequence of energies $\hat{E}(\XCoef^{1},\pi^{1}) \rightarrow \hat{E}(\XCoef^{2},\pi^{1}) \rightarrow \hat{E}(\XCoef^{2},\pi^{2}) \rightarrow \ldots$ is non-increasing and converges.
		\end{proposition}
		\begin{proof}
			Since $\XCoef^{k}$ is feasible when determining $\XCoef^{k+1}$, one has $\hat{E}(\XCoef^{k+1},\pi^{k}) \leq \hat{E}(\XCoef^{k},\pi^{k})$.
			Likewise $\pi^{k}$ is a feasible point for computing $\pi^{k+1}$ so $\hat{E}(\XCoef^{k+1},\pi^{k+1}) \leq \hat{E}(\XCoef^{k+1},\pi^{k})$. Hence, the sequence of energies is non-increasing.
			As $\hat{E}$ is bounded from below, the sequence of energies must converge.
		\end{proof}	
	
		Unfortunately this cannot be extended to modes with non-zero divergence, as changing $\XCoef_{\tn{s}}$ changes the feasible set $\segcoupl\big(Y,(1+\XCoef_{\tn{s}})^2 \cdot \mu\big)$ for $\pi$. Thus $\pi^{k}$ need not be feasible for the problem that determines $\pi^{k+1}$ and the sequence of energies created may be increasing.
		We will provide a workaround	 for this in the next section (Remark \ref{rem:OptimizationAlternatingAndScale}).
		
		The alternating scheme \eqref{eq:OptimizationAlternating} is fast and tends to converge after few iterations. But obviously it need not converge to a global optimum and the result depends on the initialization $\XCoef^1$.
		Therefore, similar to contour based segmentation functionals it must be applied with care.
		In practice application to `large' transformations, e.g.~translations and rotations, works only if a good initial guess is available (see Fig.~\ref{fig:6}).
		On the other hand it achieves decent results on smaller transformations, as most statistically learned deformations are.
				
	\subsection{Globally Optimal Branch and Bound}
		\label{sec:OptimizationBnB}
		For handling large displacement transformations, one needs a global optimization scheme.
		As discussed in Remark \ref{rem:WassersteinModesBnB}, for fixed $\XCoef$ we can eliminate $\nu$ by a separate convex optimization. One obtains \eqref{eq:WassersteinModesXCoefFunctional}:
		\begin{align}
			E_1(\XCoef) = \inf_{\nu \in \segmeas(Y,M)} & E(\XCoef,\nu) \nonumber \\
			= \inf_{\nu \in \segmeas(Y,M)} \frac{1}{2} & \inf_{\pi \in \Pi(\mu,\nu)}
				\int_{X \times Y} \left( \vphantom{\sum}
					\cGeo\big(\XEmb_\XCoef(x),y\big) + c_\feat(f_x,f_y)
				\right)\,d\pi(x,y) \nonumber \\
				& \qquad + F(\XCoef) + G(\nu)
				\label{eq:OptimizationBnBE1}
		\end{align}		
		This function is in general non-convex but low dimensional. We thus strive for a non-convex global optimization scheme.

		Given Remark \ref{rem:CouplingReformulation} $E_1(\XCoef)$ can be written as
		\begin{align}	
			\label{eq:OptimizationBnBE1Joint}
			E_1(\XCoef) = \inf_{\pi \in \segcoupl(Y,\mu)} \frac{1}{2}
				\int_{X \times Y} \left( \vphantom{\sum}
					\cGeo\big(\XEmb_\XCoef(x),y\big) + c_\feat(f_x,f_y)
				\right)\,d\pi(x,y) + F(\XCoef) + G({\project_Y}_\sharp \pi)\,.
		\end{align}
		If $G$ is zero, then by inserting suitable dummy nodes, computing $E_1(\XCoef)$ can be written as an optimal transport problem for which efficient solvers are available.

		In this section we will consider a hierarchical \emph{branch and bound} approach. We will compute lower bounds for $E_1$ on whole intervals of $\XCoef$-configurations for successively refined intervals.
		Let $\XCoefSet \subset \R^n$ be a set of $\XCoef$-values. We assume for now that all modes have zero divergence.
		For such subsets define
		\begin{align}
			E_2(\XCoefSet) = \inf_{\pi \in \segcoupl(Y,\mu)} & \frac{1}{2}
				\int_{X \times Y} \left( \vphantom{\sum}
					\left(
						\inf_{\XCoef \in \XCoefSet}
						\cGeo\big(\XEmb_\XCoef(x),y\big)
					\right) + c_\feat(f_x,f_y)
				\right)\,d\pi(x,y) \nonumber \\ 
				& \qquad + \inf_{\XCoef \in \XCoefSet} F(\XCoef) + G({\project_Y}_\sharp \pi)
				\label{eq:OptimizationBnBE2}
		\end{align}
		where we have again merged the nested optimizations as above. All occurrences of $\XCoef$ are optimized separately and independently over $\XCoefSet$.
		By introducing a nested sequence of feasible sets
		\begin{equation}
			\Lambda_{1} \supset \Lambda_{2} \supset \dotsb \supset \Lambda_n
		\end{equation}
		we obtain an adaptive convex relaxation of $E_1(\XCoef)$ over $\XCoefSet$. The relaxation becomes tighter as the set becomes smaller.
		For application in a branch and bound scheme the following properties are required:
		\begin{proposition}[\protect{\cite[Prop.~1]{SchmitzerSchnoerr-EMMCVPR2013}}]
			%\label{thm:E2}
			\thlabel{thm:E2}
			The functional $E_2$ has the following properties:
			\begin{enumerate}[(i)]
				\item $E_2(\XCoefSet) \leq E_1(\XCoef) \fa \XCoef \in \XCoefSet$,
					\label{item:BnBE2LowerBound}
				\item $\lim_{\XCoefSet \rightarrow \{\XCoef_0\}} E_2(\XCoefSet)= E_1(\XCoef_0)$,
					\label{item:BnBE2Limit}
				\item $\XCoefSet_1 \subset \XCoefSet_2 \Rightarrow E_2(\XCoefSet_1) \geq E_2(\XCoefSet_2)$.
					\label{item:BnBE2Tightness}
			\end{enumerate}
		\end{proposition}
		\begin{proof}
			Property (\ref{item:BnBE2LowerBound}): For any $\XCoef \in \XCoefSet$ obviously
			\begin{align}
				\inf_{\XCoef' \in \XCoefSet} \cGeo\big(\XEmb_{\XCoef'}(x),y\big) \leq  \cGeo\big(\XEmb_{\XCoef}(x),y\big) \qquad
				\text{and}
				\qquad \inf_{\XCoef' \in \XCoefSet} F(\XCoef') \leq F(\XCoef)\,.
			\end{align}
			So for any fixed $\pi \in \segcoupl(Y,\mu)$ (overriding the minimization in (\ref{eq:OptimizationBnBE1Joint},\ref{eq:OptimizationBnBE2})) have
			$E_2(\XCoefSet) \leq E_1(\XCoef)$. Consequently this inequality will also hold after minimization w.r.t.~$\pi$.
			
			For the limit property (\ref{item:BnBE2Limit}) note that the functions $\cGeo\big(\XEmb_\XCoef(x),y\big)$ and $F(\XCoef)$ are continuous functions of $\XCoef$. Hence, when $\XCoefSet \rightarrow \{\XCoef_0\}$ all involved minimizations will converge towards the respective function values at $\XCoef_0$ and $E_2$ converges as desired.
			
			For the hierarchical bound property (\ref{item:BnBE2Tightness}) note that for fixed $\pi$ in \eqref{eq:OptimizationBnBE2} minimization over the larger set $\XCoefSet_2$ will never yield the larger result for all occurrences of $\XCoef$. This relation will then also hold after minimization.
		\end{proof}
		
		With the aid of $E_2$ one can then construct a branch and bound scheme for optimization of $E_1$.
		Let
		\begin{align}
			\label{eq:OptimizationCoverSet}
			L = \{ (\XCoefSet_i, b_i) \}_{i \in \{1,\ldots,k\}}
		\end{align}
		be a finite list of $\XCoef$-parameter sets $\XCoefSet_i$ and lower bounds $b_i$ on $E_1$ on these respective sets. For such a list consider the following refinement procedure:
		\vskip 0.25cm
		\noindent \textbf{\texttt{refine}(L):}
		\begin{enumerate}[(1)]
			\item Find the element $(\XCoefSet_{i^\ast},b_{i^\ast}) \in L$ with the smallest lower bound $b_{i^\ast}$.
			\item Let $\subdiv(\XCoefSet_{i^\ast})=\{\XCoefSet_{i^\ast,j}\}_j$ be a subdivision of the set $\XCoefSet_{i^\ast}$ into smaller sets.
			\item Compute $b_{i^\ast,j} = E_2(\XCoefSet_{i^\ast,j})$ for all $\XCoefSet_{i^\ast,j} \in \subdiv(\XCoefSet_{i^\ast})$.
			\item Remove $(\XCoefSet_{i^\ast},b_{i^\ast})$ from $L$ and add $\{(\XCoefSet_{i^\ast,j},b_{i^\ast,j})\}_{j}$ for $\XCoefSet_{i^\ast,j} \in \subdiv(\XCoefSet_{i^\ast})$.
		\end{enumerate}
		
		\noindent This allows the following statement:
		\begin{proposition}[\protect{\cite[Prop.~2]{SchmitzerSchnoerr-EMMCVPR2013}}]
			%\label{thm:refinement}
			\thlabel{thm:refinement}
			Let $L$ be a list of finite length. Let the subdivision in \tn{\texttt{refine}} be such that any set will be split into a finite number of smaller sets, and that any two distinct points will eventually be separated by successive subdivision. Set $\subdiv(\{\XCoef_0\}) = \{ \{\XCoef_0\} \}$.
			Then repeated application of \tn{\texttt{refine}} to the list $L$ will generate an adaptive piecewise constant underestimator of $E_1$ throughout the union of the sets $\XCoefSet$ appearing in $L$.
			The sequence of smallest lower bounds will converge to the global minimum of $E_1$.
		\end{proposition}
		\begin{proof}
			Obviously the sequence of smallest lower bounds is non-decreasing and never greater than the minimum of $E_1$ throughout the considered region (see \thref{thm:E2} (\protect\ref{item:BnBE2Tightness}) and (\protect\ref{item:BnBE2LowerBound})). So it must converge to a value which is at most this minimum.
			Assume that $\{ \XCoefSet_i \}_i$ is a sequence with $\XCoefSet_{i+1} \in \subdiv(\XCoefSet_i)$ such that $E_2(\XCoefSet_i)$ is a subsequence of the smallest lowest bounds of $L$ (there must be such a sequence since $L$ is finite). Since $\subdiv$ will eventually separate any two distinct points, this sequence must converge to a singleton $\{\XCoef_0\}$ and the corresponding subsequence of smallest lowest bounds converges to $E_2(\{\XCoef_0\}) = E_1(\XCoef_0)$. Since the sequence of smallest lowest bounds converges, and the limit is at most the minimum of $E_1$, $E_1(\XCoef_0)$ must be the minimum.
		\end{proof}
		When the global optimum is unique, one can see that there also is a subsequence of $\XCoef$-sets, converging to the global optimum.

		In practice we start with a coarse grid of hypercubes covering the space of reasonable $\XCoef$-parameters (e.g.~translation throughout the image, rotation within bounds where the approximation is valid and the deformation-coefficients in ranges according to the statistical model) and the respective $E_2$-bounds. Any hypercube with the smallest bound will then be subdivided into equally sized smaller hypercubes, leading to an adaptive $2^n$-tree cover on the considered parameter range.
	
		The refinement is stopped, when the interval with the lowest bound has edge lengths that correspond to an uncertainty in $T_\XCoef(x)$ which is in the range of the discretization of $X$ and $Y$. Further refinement would only reveal structure determined by rasterization effects.

		\begin{remark}[Combining hierarchical and alternating optimization]
			\label{rem:OptimizationCombine}
			The optimum of $E_1$ w.r.t. modes that have large displacements (such as translation and rotation) tends to be rather distinct, i.e.~there is a small, steep basin around the optimal position. The hierarchical optimization scheme then works rather efficiently.
			
			On the other hand, modes that model smaller, local displacements (e.g.~those learned from training samples), often have broad, shallow basins around the optimal value. The branch and bound scheme can then take longer to converge.
			
			Therefore it suggests itself to combine the two optimization schemes: the hierarchical approach is used to determine a good initial guess for translation, rotation and a coarse estimate for the smaller modes. For this the alternating scheme is not applicable due to the non-convexity.
			But once the broad basin around the global optimum is located, the branch and bound scheme may become inefficient. Conversely, using the estimate of the hierarchical scheme as initialization, we can then expect that the alternating method will give reasonable results.
		\end{remark}

		\paragraph{Scale Mode.}
			In the presence of a scale mode one can define $E_{\textnormal{s},1}$ and $E_{\textnormal{s},2}$ equivalent to $E_1$ and $E_2$ with slight adaptations.
			\begin{align}
				E_{\textnormal{s},1}(\XCoef) = & \inf_{\nu \in \segmeas\big(Y,(1+\XCoef_{\textnormal{s}})^2 \cdot M\big)} E_{\textnormal{s}}(\XCoef,\nu) \nonumber \\
				= & \inf_{\pi \in \segcoupl\big(Y,(1+\XCoef_{\textnormal{s}})^2 \cdot \mu\big)}
					\frac{1}{2\,(1 + \XCoef_{\textnormal{s}})^2}
					\nonumber \\
					& \qquad \qquad \int_{X \times Y} \left( \vphantom{\sum}
						\cGeo\big(\XEmb_\XCoef(x),y\big) + c_\feat(f_x,f_y)
					\right)\,d\pi(x,y)
					+ F(\XCoef) + G({\project_Y}_\sharp \pi)
			\end{align}
			where in the second line we have merged the nested optimization over $\nu$ and $\pi$, see Remark \ref{rem:CouplingReformulation}.
			To obtain $E_{\textnormal{s},2}(\XCoefSet)$ all occurrences of $\XCoef$ will again be replaced by independent separate optimizations over $\XCoefSet$.
			To handle the dependency of the feasible set on $\XCoef_{\textnormal{s}}$ consider the following set:
			\begin{align}
				\segcoupl(Y,\mu_1,\mu_2) = \left\{ \pi \in \meas(X \times Y) \bcolon
					\mu_1 \leq {\project_X}_\sharp \pi \leq \mu_2 \wedge {\project_Y}_\sharp \pi \leq \mc{L}_Y \right\}
			\end{align}
			Obviously $\segcoupl\big(Y,(1+\XCoef_{\textnormal{s}})^2 \cdot \mu\big) \subset \segcoupl\big(Y,(1+\XCoef_{\textnormal{s,l}})^2 \cdot \mu, (1+\XCoef_{\textnormal{s,u}})^2 \cdot \mu\big)$ as long as $\XCoef_{\textnormal{s,l}} \leq \XCoef_{\textnormal{s}} \leq \XCoef_{\textnormal{s,u}}$.
			Then a possible definition of $E_{\textnormal{s},2}$ equivalent to \eqref{eq:OptimizationBnBE2} is
			\begin{align}
				E_{\textnormal{s},2\textnormal{a}}(\XCoefSet) = & \inf_{\pi \in \segcoupl\big(Y,(1+\XCoef_{\textnormal{s,l}})^2 \cdot \mu, (1+\XCoef_{\textnormal{s,u}})^2 \cdot \mu\big)}
				\left( \min_{\XCoef_{\textnormal{s}} \in [\XCoef_{\textnormal{s,l}},\XCoef_{\textnormal{s,u}}]} \frac{1}{2\,(1 + \XCoef_{\textnormal{s}})^2} \right) \nonumber \\
				& \int_{X \times Y} \left( \vphantom{\sum}
					\left(
						\inf_{\XCoef \in \XCoefSet}
						\cGeo\big(\XEmb_\XCoef(x),y\big)
					\right) + c_\feat(f_x,f_y)
				\right)\,d\pi(x,y) + 
					\inf_{\XCoef \in \XCoefSet} F(\XCoef) + G({\project_Y}_\sharp \pi)
			\end{align}
			where $\XCoef_{\textnormal{s,l}}$ and $\XCoef_{\textnormal{s,u}}$ are the infimum and supremum of $\XCoef_{\textnormal{s}}$ in $\XCoefSet$. It is easy to see that $E_{\textnormal{s},2\textnormal{a}}$ satisfies \thref{thm:E2} w.r.t.~$E_{\textnormal{s},1}$. The proof is analogous.
			
			If $G$ is zero the definition of $E_{\textnormal{s},2\textnormal{a}}$ can be improved upon. Consider the following lemma:
			\begin{lemma}
				%\label{thm:Superlinearity}
				\thlabel{thm:Superlinearity}
				For some cost function $c$ and $m>0$ let
				\begin{align}
					f(m) = \inf_{\pi \in \segcoupl(Y,m \cdot \mu)} \int_{X \times Y} c(x,y)\,d\pi(x,y)\,.
				\end{align}
				Then $f(m_2)/m_2 \geq f(m_1)/m_1$ for $m_2>m_1$.
			\end{lemma}
			\begin{proof}
				Assume $f(m_2) < (m_2/m_1) \cdot f(m_1)$ for $m_2 > m_1$ and let $\pi^\ast_2$ be an optimizer for $f(m_2)$. Then $(m_1/m_2) \cdot \pi^\ast_2$ is feasible for computation of $f(m_1)$ and one has
				\begin{align}
					\frac{m_1}{m_2} \int_{X \times Y} c(x,y)\,d\pi^\ast_2(x,y) = \frac{m_1}{m_2} f(m_2) < f(m_1)
				\end{align}
				which is a contradiction.
			\end{proof}
			With the aid of \thref{thm:Superlinearity} one then finds that the following is a suitable variant of $E_{\textnormal{s},2\tn{a}}$:			
			\begin{align}
				E_{\textnormal{s},2\textnormal{b}}(\XCoefSet) = & \inf_{\pi \in \segcoupl\big(Y,(1+\XCoef_{\textnormal{s,l}})^2 \cdot \mu\big)}
				\frac{1}{2\,(1 + \XCoef_{\textnormal{s,l}})^2} \nonumber \\
				& \qquad \int_{X \times Y} \left( \vphantom{\sum}
					\left(
						\inf_{\XCoef \in \XCoefSet}
						\cGeo\big(\XEmb_\XCoef(x),y\big)
					\right) + c_\feat(f_x,f_y)
				\right)\,d\pi(x,y) + 
					\inf_{\XCoef \in \XCoefSet} F(\XCoef)
			\end{align}
			The advantages over $E_{\textnormal{s},2\textnormal{a}}$ are a tighter scaling factor and a simpler feasible set for the optimal transport term.
			\begin{remark}[Scale mode and alternating optimization]
				\label{rem:OptimizationAlternatingAndScale}
				The alternating optimization scheme presented in Sect.~\ref{sec:OptimizationAlternating} only works with zero-divergence modes. The hierarchical optimization scheme can be used to extend this to the scale mode.
				The non-scale coefficients are determined by separate optimization as before, see \eqref{eq:OptimizationAlternatingXCoef}.
				The new coefficient $\XCoef_{\textnormal{s}}^{k+1}$ and $\pi^{k+1}$ are jointly determined by global hierarchical optimization, while keeping the other mode coefficients fixed (this replaces \eqref{eq:OptimizationAlternatingNu}). This hierarchical scheme will only go over one degree of freedom and thus be very quick.
				Again one finds a non-increasing sequence that must eventually converge.
			\end{remark}

	\subsection{Graph Cut Relaxation}
		\label{sec:OptimizationGraphCut}
		Both alternating and hierarchical optimization require solving a lot of optimal transport problems. Even with efficient solvers this will quickly become computationally expensive as the size of $X$ and $Y$ or the number of modes increases.
		If $G$ is non-zero then usually even more so because dedicated optimal transport solvers can no longer be applied directly to compute $E_1(\XCoef)$.
		Therefore, in this section we present a mass-constraint relaxation that, for suitable choice of $G$, turns computation of $E_1(\XCoef)$ into a min-cut problem. This can be solved very fast with dedicated algorithms and therefore the relaxation yields a huge speed-up.
		
		Throughout this section let $X$ and $Y$ be discrete sets, e.g.~pixels or super-pixels. The Lebesgue measure on $Y$ is approximated by
		\begin{align}
			\mc{L}_Y(\sigma) = \sum_{y \in \sigma} m_y
		\end{align}
		for subsets $\sigma \subset Y$, where $m_y$ is the area of super-pixel $y$.
		Any $\nu \in \segmeas(Y,M)$ can then be expressed as
		\begin{align}
			\label{eq:GraphCutUNu}
			\nu(\sigma) = \sum_{y \in \sigma} m_y\,u_\nu(y)
		\end{align}
		for all $\sigma \subset Y$ with some function $u_\nu \,\colon\, Y \rightarrow [0,1]$.
		Let $G$ be a total-variation-like local boundary regularizer of $\nu$, expressed in terms of $u_\nu$:
		\begin{align}
			\label{eq:GraphCutG}
			G(\nu) = \sum_{(y,y') \in \adjY} a_{y,y'} \cdot |u_\nu(y)-u_\nu(y')|
		\end{align}
		where $\adjY$ is the set of super-pixel neighbours and $a_{y,y'}$ is a weight that models the likelihood of a boundary between neighbours $y$ and $y'$. Such weights can be constructed from feature dissimilarity in $y,y'$, from the response of edge detectors and from the length of the boundary.

		We now relax the template-marginal constraint from the coupling set $\Pi(\mu,\nu)$ and allow $\nu$ to have arbitrary mass. So the feasible set of $\nu$ will be
		\begin{align}
			\segmeas(Y) = \left\{ \vphantom{\sum}
				\nu \in \meas(Y) \bcolon
				\nu \leq \mc{L}_Y
				\right\}\,.
		\end{align}
		This is \eqref{eq:SegmentationMeasure} without the mass constraint. The `couplings' $\pi$ will be taken from the set
		\begin{align}
			\hat{\Pi}(\nu) = \left\{ \vphantom{\sum}
				\pi \in \meas(X \times Y) \bcolon
				{\project_Y}_\sharp \pi = \nu
			\right\}\,.
		\end{align}
		Merging optimizations (see Remark \ref{rem:CouplingReformulation}) yields the feasible set
		\begin{align}
			\segcoupl(Y) = \left\{ \vphantom{\sum}
				\pi \in \meas(X \times Y) \bcolon
				{\project_Y}_\sharp \pi \leq \mc{L}_Y
			\right\}\,.			
		\end{align}
		The relaxed equivalent of $E_1$ \eqref{eq:OptimizationBnBE1Joint} that we consider in this section is
		\begin{align}
			\label{eq:GraphCutE1}
			E_{\textnormal{r},1}(\XCoef) =
			\inf_{\pi \in \segcoupl(Y)} \frac{1}{2}
				\int_{X \times Y} \left( \vphantom{\sum}
					\cGeo\big(\XEmb_\XCoef(x),y\big) + c_\feat(f_x,f_y)
				\right)\,d\pi(x,y) + F(\XCoef) + G({\project_Y}_\sharp \pi)\,.
		\end{align}
		Let $\pi^\ast$ be an optimizer of $E_{\textnormal{r},1}(\XCoef)$ for some configuration $\XCoef$. If $({\project_Y}_\sharp \pi^\ast)(y) > 0$ for some $y \in Y$, this mass will come from the cheapest $x \in X$ for this $y$, since there is no longer any constraint on the mass on $X$.
		The linear matching in the first term simplifies to a nearest neighbour matching for each $y \in Y$. This implies that the minimization in \eqref{eq:GraphCutE1} over $\pi \in \segcoupl(Y)$ can be simplified to a minimization over $\nu \in \segmeas(Y)$.
		Therefore \eqref{eq:GraphCutE1} is equivalent to
		\begin{align}
			\label{eq:GraphCutE1Nu}
			E_{\textnormal{r},1}(\XCoef) & =
			\inf_{\nu \in \segmeas(Y)} \frac{1}{2}
				\sum_{y \in Y} \cMin(y,\XCoef)\,\nu(y) + F(\XCoef) + G(\nu) \\
			\intertext{with}
			\label{eq:GraphCutCMin}
			\cMin(y,\XCoef) & = \min_{x \in X} \left( \cGeo\big(\XEmb_\XCoef(x),y\big) + c_\feat(f_x,f_y) \right)\,.
		\end{align}
		We express now $\nu$ in terms of $u_\nu$, see \eqref{eq:GraphCutUNu}, and plug in the form of the regularizer $G$ \eqref{eq:GraphCutG}. This yields
		\begin{align}
			\label{eq:GraphCutE1U}
			E_{\textnormal{r},1}(\XCoef) & =
			\inf_{u : Y \rightarrow [0,1]} \frac{1}{2}
				\sum_{y \in Y} \cMin(y,\XCoef)\cdot m_y \cdot u(y) + F(\XCoef) + \sum_{(y,y') \in \adjY} a_{y,y'} \cdot |u(y)-u(y')|\,.
		\end{align}
		For fixed $\XCoef$ this is a convex formulation of the max-flow\,/\,min-cut problem with nodes $Y$ and edges $\adjY$. The edge-weight between $y \in Y$ and the sink is given by $\cMin(y,\XCoef) \cdot m_y$ and the weights of the edges between $y,y' \in Y$ by $a_{y,y'}$. This problem can be solved very efficiently by dedicated algorithms, see for example \cite{KolmogorovGraphCutExperiments}.

		\begin{remark}[Optimization of $E_{\textnormal{r},1}$]
			Both the alternating method and the hierarchical scheme, Sects.~\ref{sec:OptimizationAlternating} and \ref{sec:OptimizationBnB}, can be applied directly to the optimization of $E_{\textnormal{r},1}$.
			The sequence equivalent to \eqref{eq:OptimizationAlternating} will provide a non-increasing converging sequence of energies. Since the dependence of the feasible set on the mass of $\mu$ has disappeared, it can also be extended to the scale mode.
			Also, handling the scale mode in the hierarchical scheme is simplified.
		\end{remark}	
						
		Functional \eqref{eq:GraphCutE1U} can be interpreted as a binary Markov random field (MRF) with labels fore- and background ($u \in \{1,0\}$) and a latent object configuration variable $\XCoef$.
		Such enhanced MRFs have been used in \cite{Kumar05} with the latent variables describing layered pictorial structures and in \cite{YangelShapeGraphMRF} with graph-based shape models.
		Optimization of a general class of such models via branch and bound has been discussed in \cite{LempitskyBranchAndMincut2008}.
		A main difference of the approach presented here and \cite{LempitskyBranchAndMincut2008} is that the shape variations are not captured implicitly in the hierarchical cluster of sample shapes but explicitly and smoothly in the set of learned Wasserstein modes.

\section{Numerical Examples}
	\label{sec:Experiments}
	We will now present some numerical examples for joint image segmentation and shape matching with Wasserstein modes. The scope of these examples is to transparently show the key properties of the functional (geometric invariance, response to noisy data etc.) and to demonstrate its applicability to different types of geometric data and features.

	\subsection{Setup and Implementation Details}
	\paragraph{Setting up the Model.} As discussed in Sect.~\ref{sec:GeometricInvariance} the functional component $F$, modelling the distribution of the deformation parameter $\XCoef$ (c.f.~\eqref{eq:WassersteinModesFunctional}), was not depending on the $\XCoef$-entries that describe translation, rotation and scale. For the statistical modes we modelled a simple Gaussian as given by \eqref{eq:FDefinition}. The weight $\gamma$ was set to a small value, i.e.~we `trusted' the data for small deformations and mainly wanted to keep the deformations from becoming too large, where the linear deformation model does no longer work very well.
	
	The number of used modes ranged between 3 and 8 for branch and bound, up to about 14 for the alternating scheme. As discussed in Sect.~\ref{sec:OptimizationBnB}, for the initial covering $L$ of the parameter space for $\XCoef$, we used a grid of $n$-dimensional hypercubes: for the translation components ranging over the area of the image, for rotation and scale within the limits where the numerical approximation is valid and for the statistical modes depending on the observed standard deviations during learning.
	
	A very important parameter in the functional is the relative weight between the geometric and the appearance cost function, $\cGeo$ and $c_\feat$. When the appearance features are very noisy, we tend to put more trust on the geometric component and thus the predefined deformation modes. For very reliable data we may accept a previously unknown deformation to better match the observed features. Some intuition on how to choose this relative weight may be gained from Fig.~\ref{fig:5}.

	\paragraph{Optimization Algorithms.} In most experiments numerical optimization was carried out in two steps, starting with branch and bound over the modes with largest deformations, followed by alternating optimization over all modes (see Remark \ref{rem:OptimizationCombine}). As pointed out in Remark \ref{rem:CFeatAndMetric} continuous solvers cannot be applied since the marginal $\nu$ is unknown. Therefore we rely on discrete algorithms. For numerical optimization of $E_2(\XCoefSet)$ we implemented two different methods:
	\begin{itemize}
		\item For $G = 0$, i.e.~in the absence of an additional segmentation term on the marginal $\nu$ (c.f.~\eqref{eq:WassersteinModesFunctional}), the functional $E_2(\XCoefSet)$ \eqref{eq:OptimizationBnBE2} can be evaluated by using a dedicated optimal transport solver. For this we wrote a $\texttt{c++}$ implementation of the Hungarian method \cite{KuhnHungarianMethod}.
		\item When $G$ is a discrete total-variation-like local regularity prior (see Sect.~\ref{sec:OptimizationGraphCut}, eq. \eqref{eq:GraphCutG}) evaluation of $E_2(\XCoefSet)$ can be written as a linear program, which we solved with \texttt{CPLEX}.
	\end{itemize}
	The top level, i.e.~everything except for the calls to optimize $E_2(\XCoefSet)$ was implemented in Mathematica.

	In practice we used the first variant for the branch and bound stage and the second variant for the subsequent alternating optimization stage. The reasoning behind this is that the total variation of a segmentation depends mostly on its local properties and can be vary significantly without altering its global configuration, which is what we look for during the branch and bound optimization. TV is then added during the `fine-tuning' in the alternating stage.

	\paragraph{Reducing Complexity in Practice.} To reduce computational complexity, we sampled the cost function $\cGeo(x,y) + c_\feat(f_x,f_y)$ for fixed $x$ only at positions $y$ close to $x$. When $y$ is very far from $x$ the high geometric cost will make the assignment very unlikely. The cut-off radius around $x$ is chosen according to the range of $c_\feat$ and the size of the mode parameter set $\XCoefSet$ during branch and bound. Global optimality of the sub-sampled cost-function w.r.t.~the dense model can be checked by introducing `overflow' variables with suitable assignment costs for each $x$: as long as no mass is put onto these overflow variables, the optimizer of the reduced model is also globally optimal in the dense model.
	
	\paragraph{Computational Complexity and Runtime.} Although we only used experimental code, which was far from being optimized for performance we briefly comment on the observed running-times to give the reader a general idea of the applicability. Experiments were performed on a standard desktop computer with an Intel Core i7 processor at 3.4 GHz and 16 GB RAM. The branch and bound scheme, which is the computationally most demanding part, was parallelized over the processor cores. The alternating optimization is much less demanding and consequently converges much faster.
	The discrete templates had several 100 points, the discrete images, super-pixel segmentations, etc. several 1000 points.
	
	For the branch and bound scheme the running time is determined by how many of the tree of bounds at different scales have to be explored until a minimizer is found.
	This number is sensitive to several factors: it grows exponentially with the number of degrees of freedom. Also, it depends on the specific problem instance and how well the global optimum is pronounced. In the presence of strong noise or multiple similarly good minima the scheme will naturally take longer as in a problem with only one distinct solution. Consequently it is not really possible to accurately estimate the number of required bounds beforehand, i.e.~to give an overall expected complexity estimate of the branch and bound scheme.
	
	During our experiments we observed running times from under a minute for 3-4 modes on `easy problems' up to about a day for 7-8 modes on very noisy and large instances.
	Instances shown in this section were mostly set up such that branch and bound would take 10 minutes at most.
	
	Of course the running time also depends strongly on the problem dimensions. Fortunately, the flexible mathematical framework provides means for reducing the problem dimensions easily by working for example on an over-segmentation with super-pixels instead of on the full pixel grid.
	The loss of resolution can often be compensated for by adding a local regularizer.

	\subsection{Numerical Results}
	We start with some synthetic experiments to transparently illustrate different properties of the functional. For these experiments the feature cost function $c_\feat(f_x,f_y)$ was chosen to be constant w.r.t.~$x$, i.e.~every template point expects the same features and the template has a homogeneous appearance. This corresponds to a classifier that tries to locally asses for each pixel whether it is part of the fore- or background.
	
	\paragraph{Branch and Bound.}
	A shape model of a bunny is learned from several different views. The subsequent task is then to find a novel view (within the range of the training views) among a collection of different shapes.
	Branch and bound was used to optimize over translations, rotation and scale of the object. On these degrees of freedom the alternating scheme is prone to getting stuck in a poor local minimum, if initialized on the wrong shape. Afterwards the alternating scheme was applied to account for non-isometric variations due to perspective. Additionally it is shown how the rotation invariance can be extended to large angles.
	The results of this experiment are illustrated in Fig.~\ref{fig:1-bunny}.
	
	\begin{figure}
		\centering
		\subfloat[]{%
			\label{fig:1-1} % unaries and optimal segmentation
			\includegraphics[width=3cm]{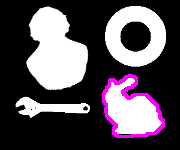}
		}%
		\subfloat[]{%
			\label{fig:1-2} % branch and bound grid
			\includegraphics[width=3cm]{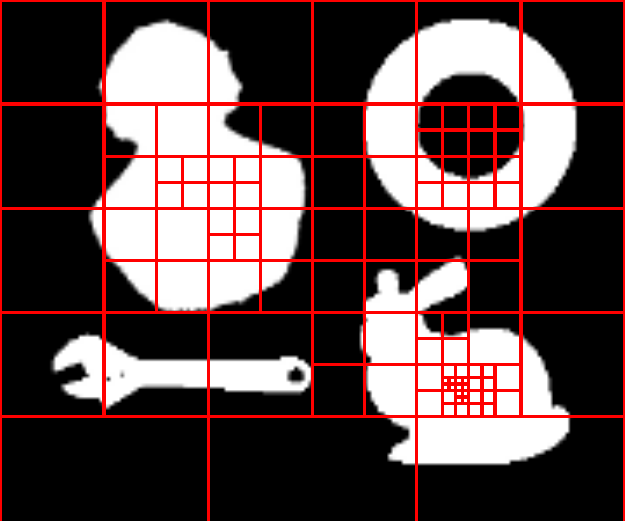}
		}%
		\subfloat[]{%
			\label{fig:1-3} % large angle rotation
			\includegraphics[width=3cm]{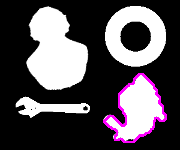}
		}%
		\hskip 0.5cm
		\subfloat[]{%
			\label{fig:1-4} % multiple foot-point illustration
			\resizebox{!}{2.5cm}{
			\begin{tikzpicture}
					[
					footpoint/.style={anchor=center, minimum width=0.2cm, shape=circle, fill=red, inner sep=0pt}
					]
				\draw (0,0) circle (3cm);
				\begin{scope}[x=3cm,y=3cm]
					\coordinate (f1) at (1,0);
					\coordinate (f2) at (0.707107,0.707107) [] {};
					\coordinate (f3) at (0,1) [] {};
					\coordinate (f4) at (-0.707107,0.707107) [] {};
					\coordinate (f5) at (-1,0) [] {};
					\coordinate (f6) at (-0.707107,-0.707107) [] {};
					\coordinate (f7) at (0,-1) [] {};
					\coordinate (f8) at (0.707107,-0.707107) [] {};
				\end{scope}
				\begin{scope}[x=2cm,y=2cm]
					\coordinate (v1) at (1,0);
					\coordinate (v2) at (0.707107,0.707107);
					\coordinate (v3) at (0,1);
					\coordinate (v4) at (-0.707107,0.707107);
					\draw ($(f1)-(v3)$) -- ($(f1)+(v3)$);
					\draw ($(f2)-(v4)$) -- ($(f2)+(v4)$);
					\draw ($(f3)-(v1)$) -- ($(f3)+(v1)$);
					\draw ($(f4)-(v2)$) -- ($(f4)+(v2)$);
					\draw ($(f5)-(v3)$) -- ($(f5)+(v3)$);
					\draw ($(f6)-(v4)$) -- ($(f6)+(v4)$);
					\draw ($(f7)-(v1)$) -- ($(f7)+(v1)$);
					\draw ($(f8)-(v2)$) -- ($(f8)+(v2)$);
				\end{scope}
				\begin{scope}
					\node[footpoint] at (f1) [] {};
					\node[footpoint] at (f2) [] {};
					\node[footpoint] at (f3) [] {};
					\node[footpoint] at (f4) [] {};
					\node[footpoint] at (f5) [] {};
					\node[footpoint] at (f6) [] {};
					\node[footpoint] at (f7) [] {};
					\node[footpoint] at (f8) [] {};
				\end{scope}
			\end{tikzpicture}
			}
		}%
		\caption{\textbf{Shape location with branch and bound.} We are searching for a bunny among a collection of other shapes via branch and bound. %
		\protect\subref{fig:1-1} Illustration of $c_\feat(f_x,f_y)$, white indicating foreground affinity. The optimal segmentation is given by the purple line. %
		\protect\subref{fig:1-2} The covering set $L$ \protect\eqref{eq:OptimizationCoverSet} upon convergence of branch and bound, projected onto the two translation components, shown relative to the query image. Very dissimilar shapes such as the wrench can be ruled out at a coarse level while more similar shapes such as the bust can only be discarded on finer scales. The grid is finest at the true location of the bunny. %
		\protect\subref{fig:1-3} Modified problem with the bunny rotated by a large angle. Such large angles cannot be covered by the rotation mode, see \protect\eqref{eq:RotationMode} and its discussion. Instead, one can use multiple support points on the shape manifold (sketched in \protect\subref{fig:1-4}), each equipped with a local rotation mode, and integrate them into the branch and bound scheme.%
		}
		\label{fig:1-bunny}
	\end{figure}
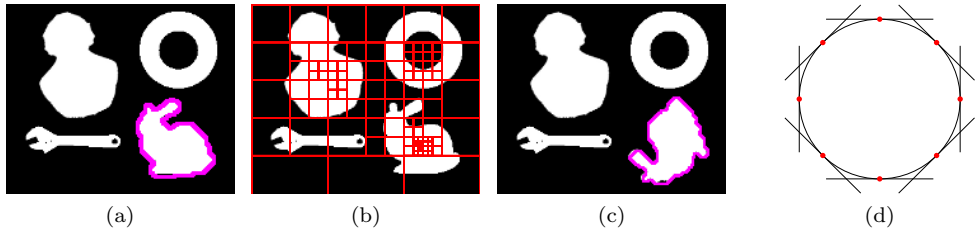

	\paragraph{Background Modelling.} Note that in order to locate the bunny correctly, we sometimes also need to model the image background in some way. This can be done implicitly by ensuring that the boundary of the foreground is aligned with detected contours in the image via a weighted TV-like term through $G(\nu)$ in \eqref{eq:WassersteinModesFunctional}. A more explicit approach is to extend the template to include a small region `looking like background' around the foreground (see Sect.~\ref{sec:BackgroundModelling}). This is demonstrated in Fig.~\ref{fig:2}.
	
	\begin{figure}
		\centering
		\subfloat[]{%
			\label{fig:2-1}
			\includegraphics[width=3.6cm]{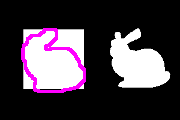}
		}%
		\hskip 0.25cm
		\subfloat[]{%
			\label{fig:2-4}
			\includegraphics[height=2.4cm]{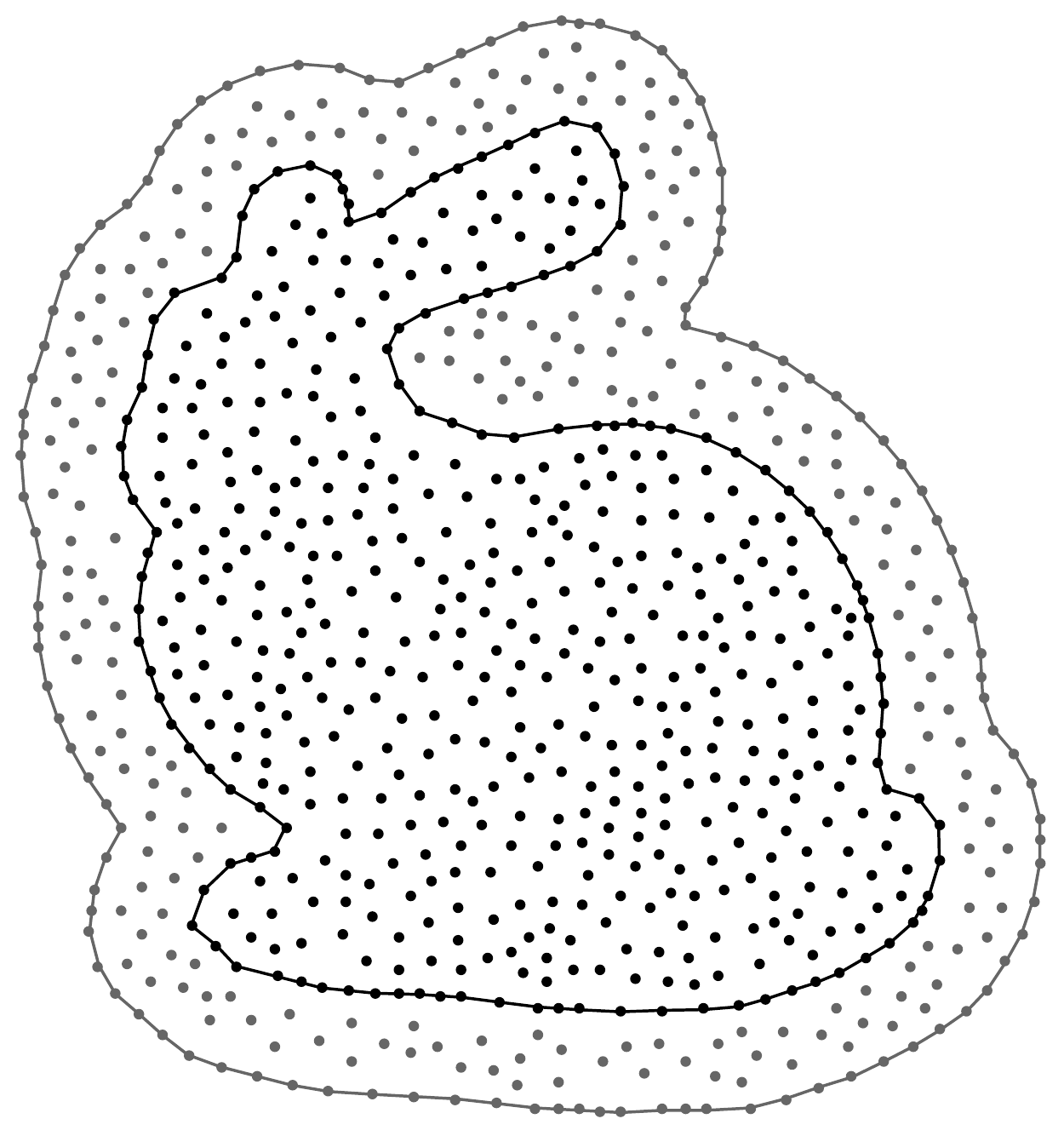}
		}%
		\hskip 0.25cm
		\subfloat[]{%
			\label{fig:2-2}
			\includegraphics[width=3.6cm]{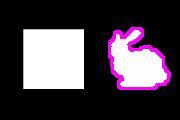}
		}%
		\subfloat[]{%
			\label{fig:2-3}
			\includegraphics[width=3.6cm]{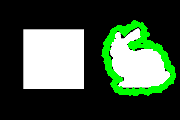}
		}%
		\caption{\textbf{Background modelling.} %
			\protect\subref{fig:2-1} Na\"ive segmentation without modelling the image background: sometimes it is then the optimal configuration to `immerse' the sought-after shape into a large blob of false-positive detections. %
			\protect\subref{fig:2-4} The shape template: to solve this problem, we can model a small area of background (gray) around the boundary of the object (black). %
			\protect\subref{fig:2-2} Optimal segmentation when the background around the object is modelled. 
			\protect\subref{fig:2-3} Region which is assigned to the explicitly modelled background.
		}
		\label{fig:2}
	\end{figure}

	More examples on detecting objects in a noisy environment and on restoring shapes from distorted detections are given in Figs.~\ref{fig:3} and \ref{fig:4}.

	\begin{figure}
		\centering
		\includegraphics[width=3.6cm]{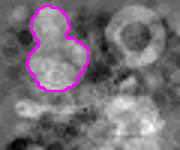}
		\hskip 1cm
		\includegraphics[width=3.6cm]{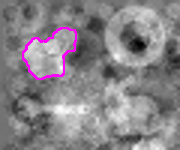}
		\hskip 1cm
		\includegraphics[width=3.6cm]{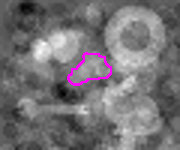}
		\caption{\textbf{Locating shapes in a noisy environment.} We are looking for the bust of Beethoven in a picture with non-local noise and other shapes present. In the first two examples the shape is correctly identified. In the third example, the true bust is missed, because it is rather small and instead a chunk of false-positive noise is segmented.}
		\label{fig:3}
	\end{figure}
	
	\begin{figure}
		\centering
		\includegraphics[width=3.cm]{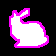}
		\hskip 1cm
		\includegraphics[width=3.cm]{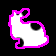}
		\hskip 1cm
		\includegraphics[width=3.cm]{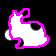}
		\caption{\textbf{Restoring distorted shapes.} By aid of the template geometry non-local noise, e.g.~partial occlusion and false-positive detections can be recognized and the segmentation retains the true sought-after shape.}
		\label{fig:4}
	\end{figure}
	
	\paragraph{Interaction of Regularizers.}
	Now let us study the interaction between the different components of the functional. Let $G$ be the discrete total variation of $\nu$ \eqref{eq:GraphCutG}. For now we ignore deformations and simply take a fixed template. That is we consider the following functional:	
	\begin{align}
		E(\nu) = \inf_{\pi \in \Pi(\mu,\nu)} \int_{X \times Y} \left(
			\|x-y\|^2 + \tau \cdot c_\feat(f_x,f_y) \right) d\pi(x,y) + \sigma \cdot G(\nu)
	\end{align}
	where we have introduced weights $\tau$ and $\sigma$. In Fig.~\ref{fig:5} it is illustrated how the optimal segmentations depend on $\tau$ and $\sigma$ in the presence of different types of noise.

	\begin{figure}
		\centering
		\begin{tabular}{ccccc}
			\includegraphics[width=2.7cm]{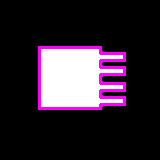} & %
			\includegraphics[width=2.7cm]{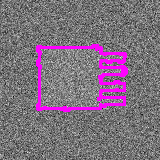} & %
			\includegraphics[width=2.7cm]{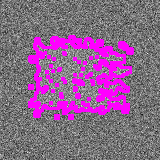} & %
			\includegraphics[width=2.7cm]{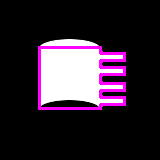} & %
			\includegraphics[width=2.7cm]{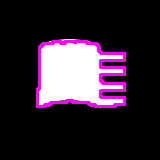} \\ %
%			\subfloat[]{\label{fig:5-1}} & %
%			\subfloat[]{\label{fig:5-2}} & %
%			\subfloat[]{\label{fig:5-4}} & %
%			\subfloat[]{\label{fig:5-6}} & %
%			\subfloat[]{\label{fig:5-8}} \\ %
			\includegraphics[width=2.7cm]{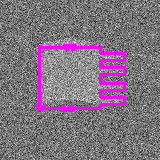} & %
			\includegraphics[width=2.7cm]{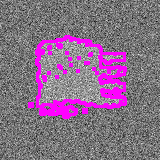} & %
			\includegraphics[width=2.7cm]{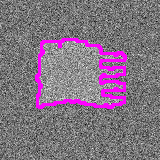} & %
			\includegraphics[width=2.7cm]{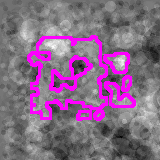} & %
			\includegraphics[width=2.7cm]{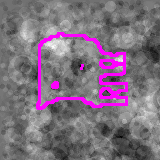} \\ %
		\end{tabular}
		\caption{\textbf{Interaction of Regularizers.} %
		\textit{Top row, from left to right:} %
		\textit{(1)} Clean problem with object formed exactly like template. %
		\textit{(2)} Local Gaussian noise, low feature-cost weight $\tau$, $\sigma=0$, i.e.~the optimal matching is dominated by the geometric cost component. %
		\textit{(3)} Same problem as before, but with a high $\tau$: now the local noise severely affects the segmentation, which becomes very irregular. %
		\textit{(4)} An unknown deformation is encountered (not described by a known deformation mode). With low $\tau$ it is ignored. %
		\textit{(5)} With a higher $\tau$ the optimal segmentation locally adapts to the unknown deformation. %
		\textit{Bottom row:} %
		\textit{(1)} Unknown deformation with local noise and low $\tau$: now the trick to simply increase $\tau$ \textit{(2)} to adapt for the unknown deformation does no longer work, as the local noise is distorting the segmentation. %
		\textit{(3)} The problem can be solved by adding a local boundary regularizer ($\sigma>0$): it helps to distinguish between the local Gaussian noise and the non-local unknown deformation. So the optimal segmentation ignores the former but adapts to the latter. %
		\textit{(4)} The same trick does not work with non-local noise: now unknown deformation and non-local noise cannot be separated and the optimal segmentation becomes faulty. %
		\textit{(5)} Adding the deformation as a Wasserstein mode helps to approximately find the object even in this noisy scenario, also thanks to the robustness of the globally optimal branch and bound scheme.
		}
		\label{fig:5}
	\end{figure}
%	figure as sketched on paper.
%	$c_feat$ + TV as usual.
%	$c_feat$ <-> $c_geo$: damp noise but lose flexibility. remedy: TV, but this also only when noise is local.
%	also: TV can compensate for coarser resolution of template.
	
	\paragraph{Alternating Optimization.}
	In Fig.~\ref{fig:6} the behaviour of the alternating optimization scheme is elucidated. In particular it becomes apparent how in noisy problems the scheme easily gets stuck in poor local minima. This is a general problem of local optimization methods and proves the importance of the globally optimal branch and bound scheme to provide a proper initial starting point.
	
	\begin{figure}
		\centering
		\subfloat[]{ \label{fig:6-1} \includegraphics[width=4cm]{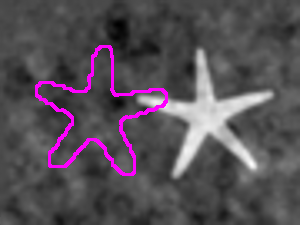}}
		\subfloat[]{ \label{fig:6-2} \includegraphics[width=4cm]{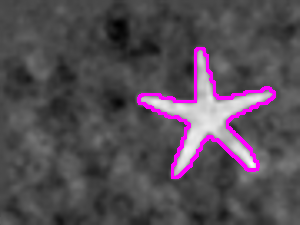}}
		\subfloat[]{ \label{fig:6-3} \includegraphics[width=4cm]{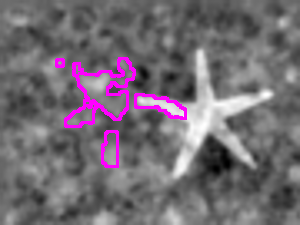}} \newline
		\subfloat[]{ \label{fig:6-4} \includegraphics[width=4cm]{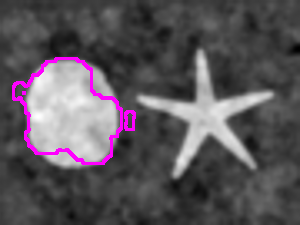}}
		\subfloat[]{ \label{fig:6-5} \includegraphics[width=4cm]{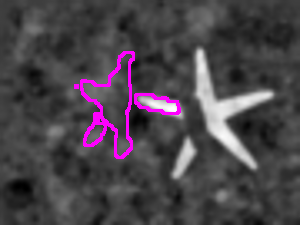}}
		\caption{\textbf{Alternating Optimization.} The fundamental limitations of the local optimization scheme become apparent in this experiment. %
		\protect\subref{fig:6-1} Initial position, some overlap with true segmentation is given. %
		\protect\subref{fig:6-2} Upon convergence the true shape has been located. %
		\protect\subref{fig:6-3} Same scenario but with a higher noise: now the alternating scheme gets stuck along the way. %
		\protect\subref{fig:6-4} A large, but not starfish-shaped blob on the left by mistake attracts the template. %
		\protect\subref{fig:6-5} The partial occlusion of the shape obstructs the convergence. The local scheme has no way of knowing `that the starfish continues' beyond the occlusion.}
		\label{fig:6}
	\end{figure}
	
	\paragraph{Super-pixels.} An important feature of functional \eqref{eq:WassersteinModesFunctional} is that its discrete version readily encompasses a wide range of data structures. As the computational complexity strongly depends on the size of the discretizations of $X$ and $Y$ it may be reasonable to apply the functional not directly to the pixel level but to a coarser over-segmentation as for example provided by super-pixels. Some examples with the class `starfish' are given in Fig. \ref{fig:7}.
	Fig.~\ref{fig:8} shows some of the involved non-isometric deformations to illustrate the range of the linear modes model and also one example where the limit of the linear expansion has been reached.

	\begin{figure}
		\centering
		\begin{tabular}{ccc}
		\includegraphics[height=3.2cm]{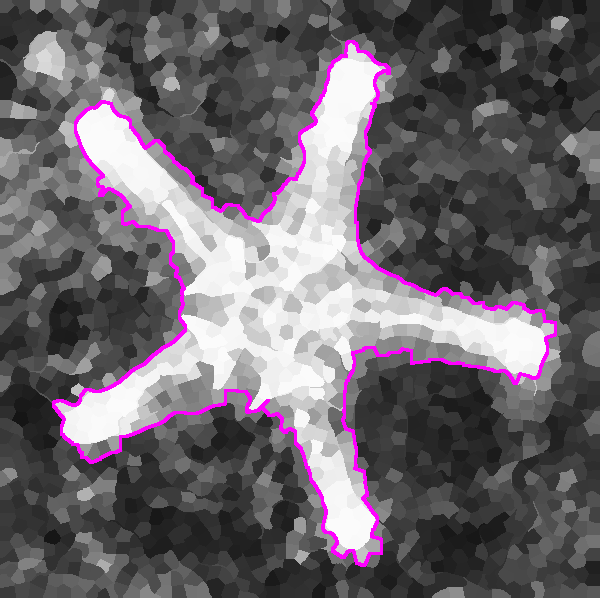} &
		\includegraphics[height=3.2cm]{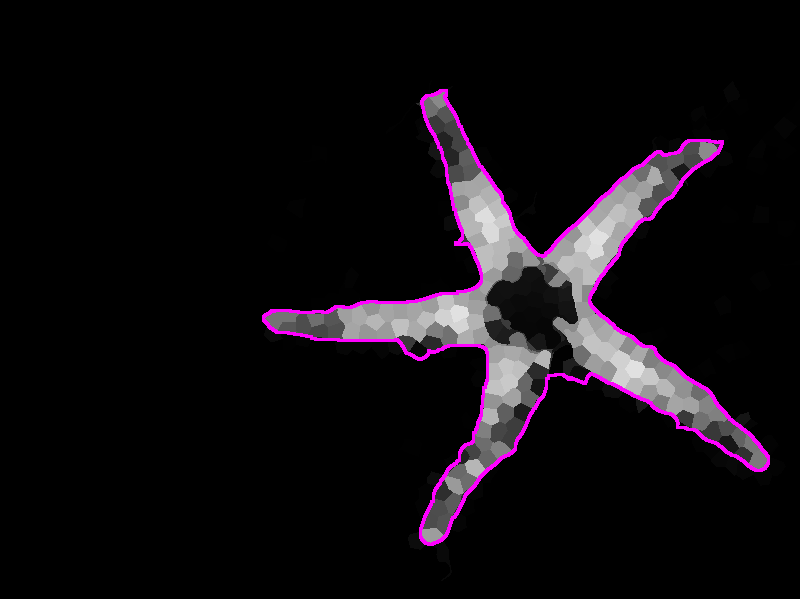} &
		\includegraphics[height=3.2cm]{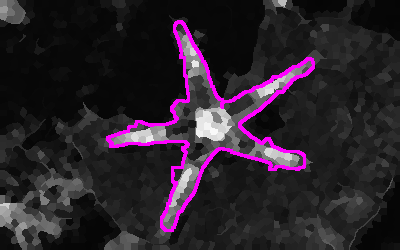} \\
		\includegraphics[height=3.2cm]{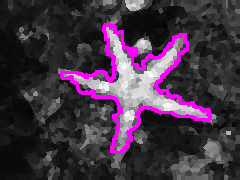} &
		\includegraphics[height=3.2cm]{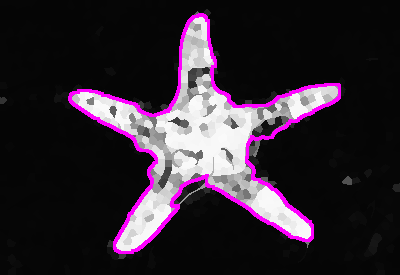} &
		\includegraphics[height=3.2cm]{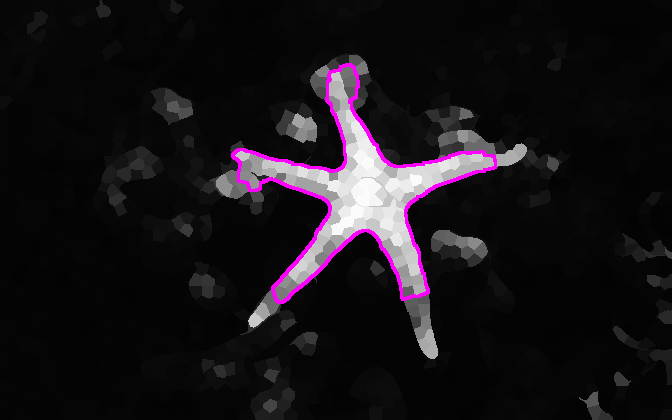} \\
		\end{tabular}
		\caption{\textbf{Application to super-pixel images.} The numerical framework extends seamlessly to super-pixel images. A simple local classifier based on color was applied to super-pixel images of starfish. The classifier was intentionally designed to yield partially faulty results. %
		With simultaneous matching and segmentation, locally faulty detections can be corrected for: false-positive clutter is ignored, missing parts are restored. On the bottom-right an example is given where the deformation modes are not flexible enough to adapt to the true object shape. }
		\label{fig:7}
	\end{figure}

	\begin{figure}
		\centering
%		\begin{tabular}{ccc}
%		\includegraphics[height=3.6cm]{fig6b-01.pdf} &
%		\includegraphics[height=3.6cm]{fig6b-02.pdf} &
%		\includegraphics[height=3.6cm]{fig6b-03.pdf} \\
%		\includegraphics[height=3.6cm]{fig6b-04.pdf} &
%		\includegraphics[height=3.6cm]{fig6b-05.pdf} &
%		\includegraphics[height=3.6cm]{fig6b-06.pdf} \\
%		\end{tabular}
%		\caption{\textbf{Range of linear deformation model.} We give some examples for non-isometric shape variations of the template from the starfish segmentation test problems. One can see that substantial changes in shape can be encoded by the linear modes. On the bottom-right an example is shown where the deformation coefficients $\XCoef$ have become too large and the shape looks distorted.}
		\begin{tabular}{ccc}
		\includegraphics[height=3.2cm]{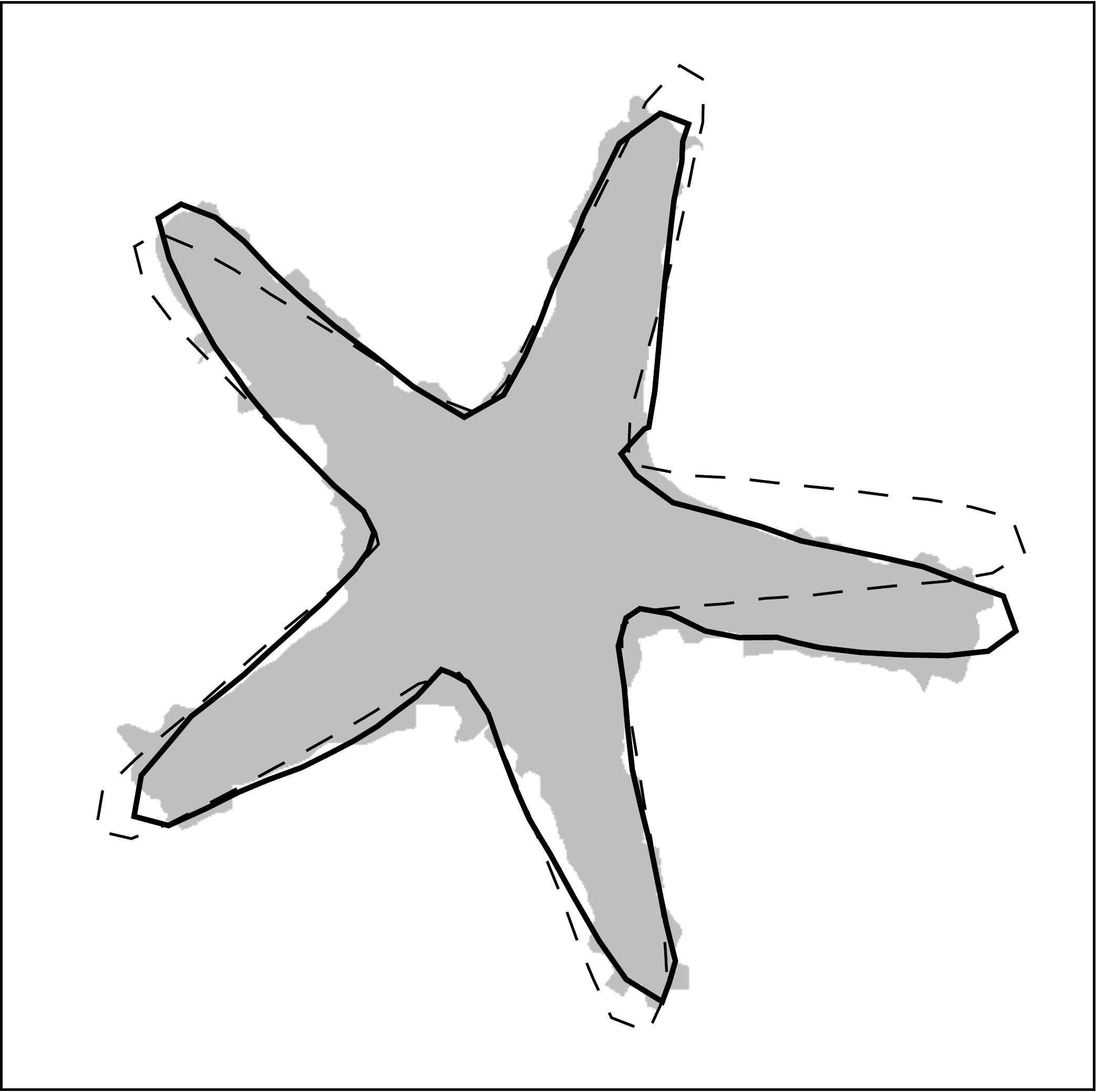} &
		\includegraphics[height=3.2cm]{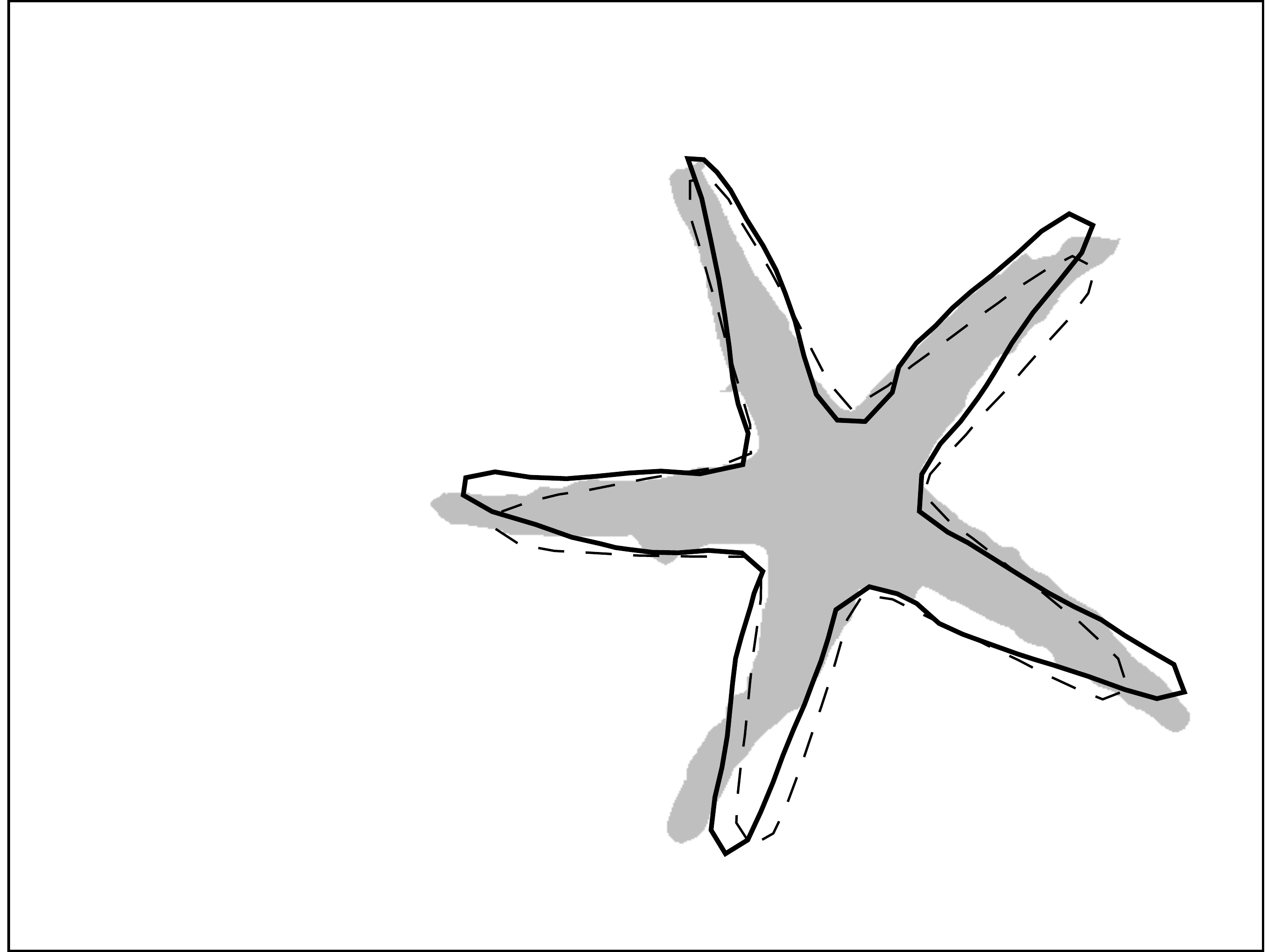} &
		\includegraphics[height=3.2cm]{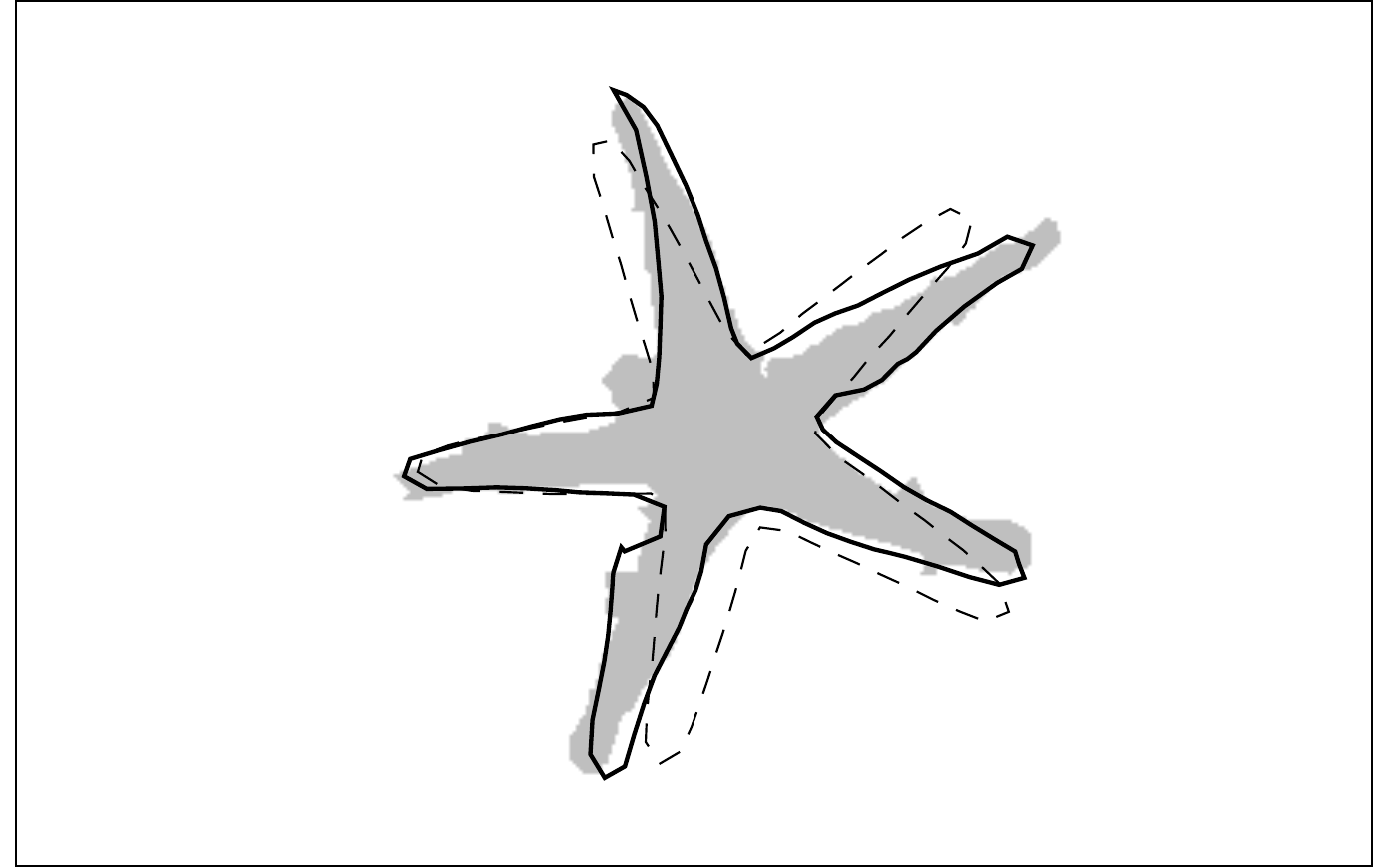} \\
		\includegraphics[height=3.2cm]{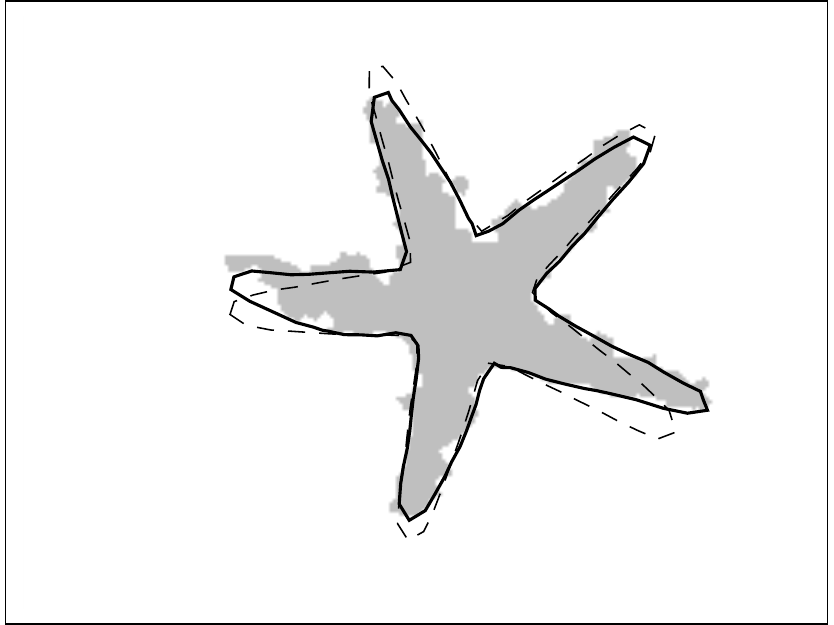} &
		\includegraphics[height=3.2cm]{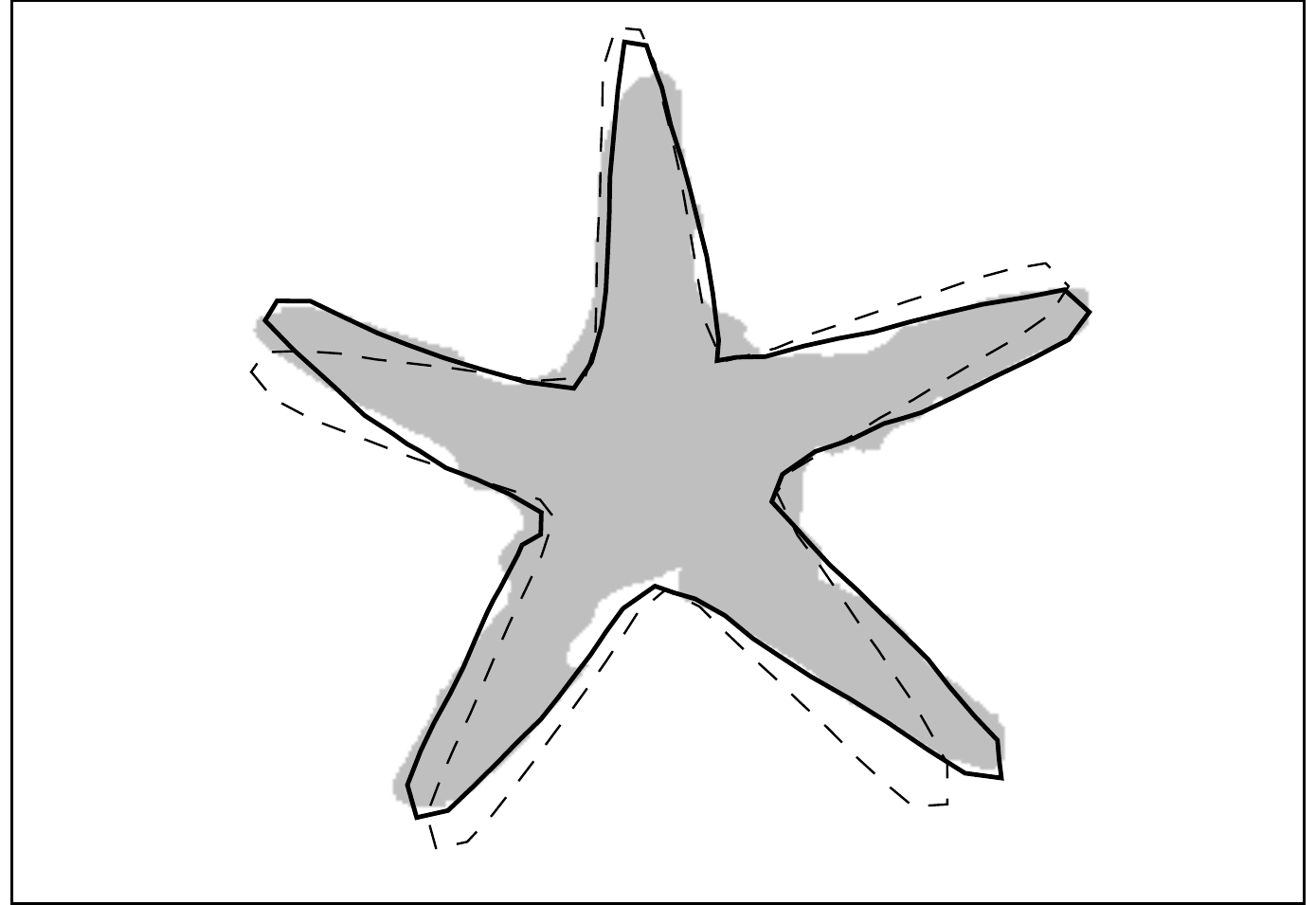} &
		\includegraphics[height=3.2cm]{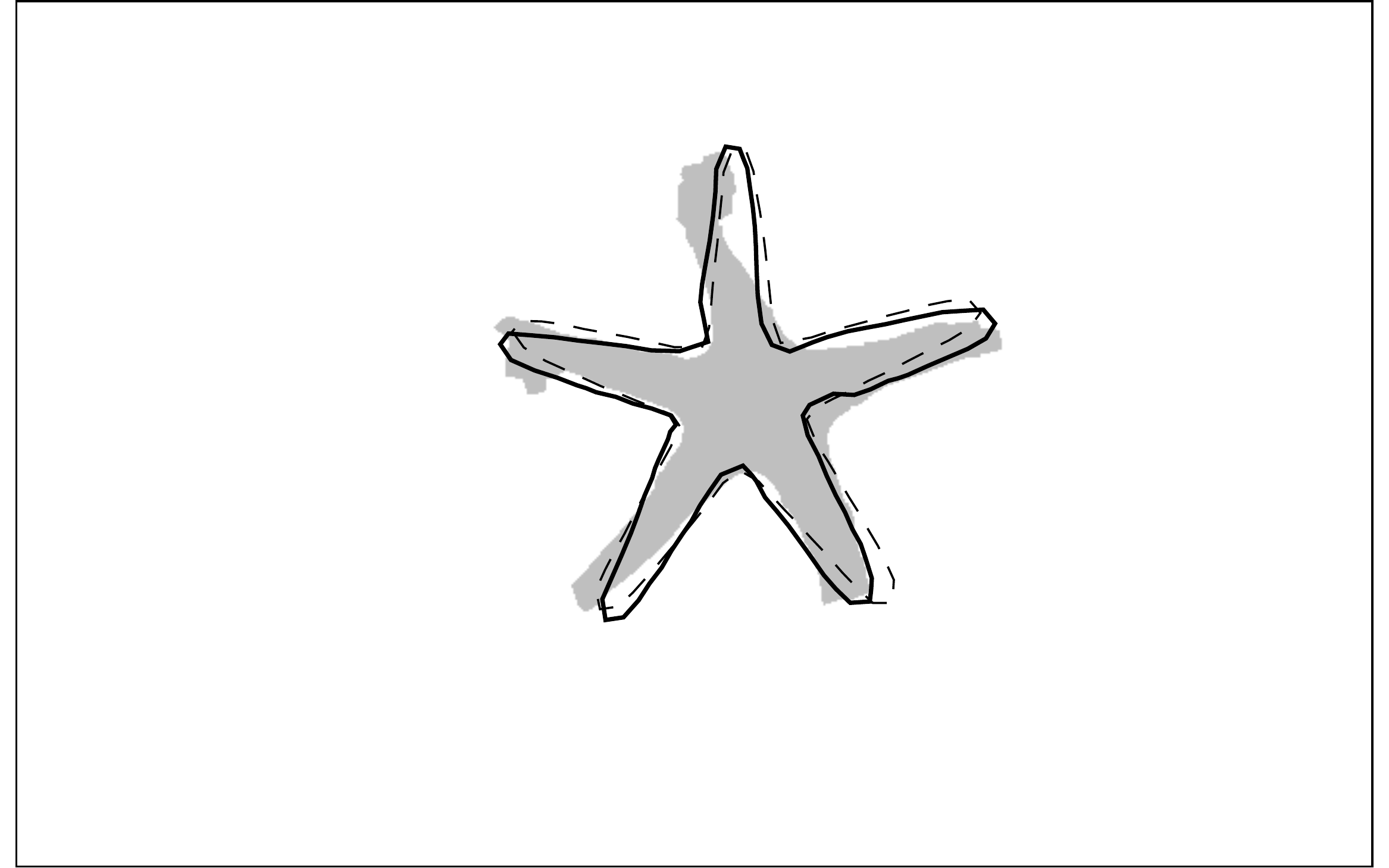} \\
		\end{tabular}
		\caption{\textbf{Range of linear deformation model.} For the segmentations in Fig.~\ref{fig:7} we illustrate here in the same order the relative configuration of the template after translation and rotation (dashed lines), the fully transformed template (black lines) and the segmentation (gray shading). The experiments used three isometric (translation + rotation) and ten statistical modes. One can see that substantial changes in shape can be encoded by the linear modes. On the top-right an example is shown where the deformation coefficients $\XCoef$ have become too large and the shape looks distorted.}
		\label{fig:8}
	\end{figure}

	In Fig.~\ref{fig:9} the scale invariance of the approach is demonstrated by actually deliberately breaking it. The same functional is optimized twice, but with a different prior on the allowed object scale. Depending on the admissible scale, once the large and once the small clownfish is segmented. Such a task can only be solved with global optimization techniques.

	\begin{figure}
		\centering
		\includegraphics[width=6cm]{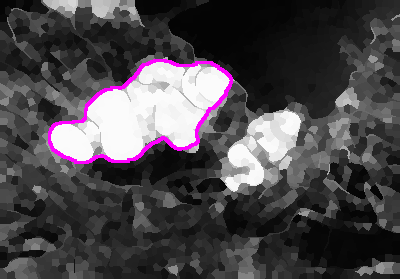}
		\hskip 2cm
		\includegraphics[width=6cm]{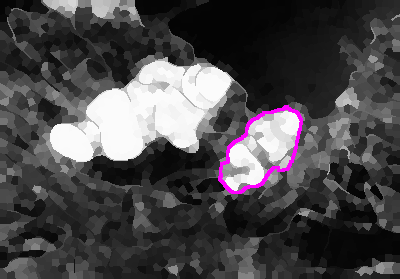}
		\caption{\textbf{Scale invariant segmentation.} Similar to the starfish experiment, a super-pixel image of two clownfish is to be segmented, based on an imperfect local color classifier. With the aid of a shape prior, by setting a preferred range of object scales, but leaving the rest of the approach scale invariant, depending on the choice, both the large and the small fish are correctly located. Note that this example also requires proper modelling of the object boundary (see Fig.~\protect\ref{fig:2}).}
		\label{fig:9}
	\end{figure}
	
	\paragraph{Inhomogeneous $c_{\feat}$.}
	So far we have only considered the case where $c_{\feat}(f_x,f_y)$ was constant w.r.t.~$x$. However, computationally there is no increase in complexity if we pick a more general feature cost. The potential of this additional freedom is now demonstrated on an example with the UIUC database (see for example \cite{UIUCDatasetPaper}). This is a set of gray level side views of parking cars.
Locating these cars cannot be approached with a homogeneous foreground\,/\,background detector, as no consistent separation based on local appearance features seems to be possible.

	Therefore we now learn local detectors for each point of the template $X$ separately and based on these compute an inhomogeneous $c_\feat$.
	As features we use local histograms of the image color and its gradient.
	We compute assignments between the learned template and the training cars (both shapes fixed, only geometric, no appearance cost). Based on these assignments we extract for each template point $x$ the collection of expected features $f_x$. Then, on a test image $Y$ we compare for each super-pixel $y \in Y$ its histogram of features $f_y$ with the distribution of expected features $f_x$ on each template point via an optimal transport based histogram distance (see e.g.~\cite{Pele2009}). These comparison costs were used as costs $c_\feat(f_x,f_y)$.

	We want to emphasize at this point that we do in no way champion this particular choice of features and this choice does not constitute a part of our presented framework. We merely seek to provide a transparent set-up to demonstrate the benefit of locally adaptive template appearance without obstruction through more complicated feature acquisition and processing.

	Fig.~\ref{fig:10} gives an impression of the functions $c_\feat(f_x,f_y)$ obtained in this way. Obviously, for a single template point $x \in X$ the associated cost is very noisy and not very informative. We can thus only hope that through the combination of all template pixels and the knowledge about their relative spatial arrangement we can identify the positions of the cars.

	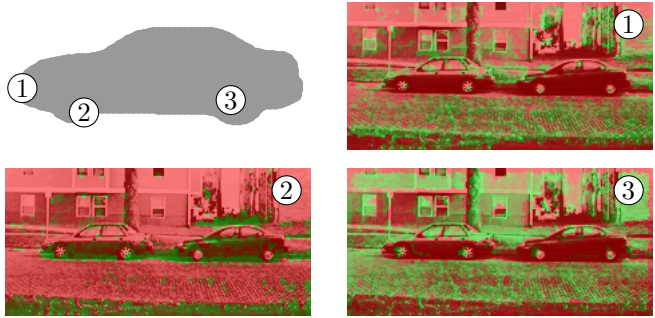
\begin{figure}
		\centering
		\begin{tikzpicture}[every node/.style={inner sep=0, anchor=north west}]
			\node at (0,0) {\includegraphics[width=4cm]{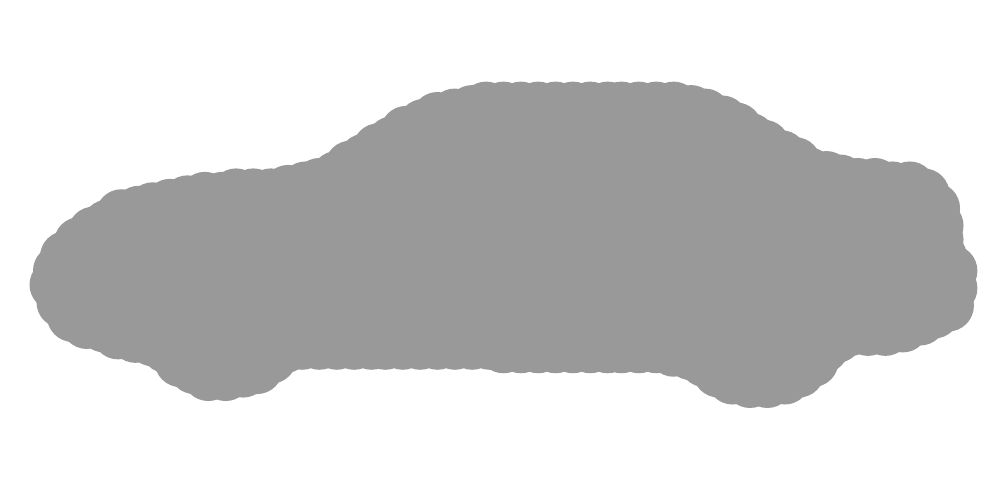}};
			\node at (4.5,0) {\includegraphics[width=4cm]{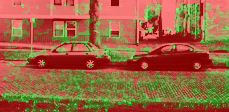}};
			\node at (0,-2.2) {\includegraphics[width=4cm]{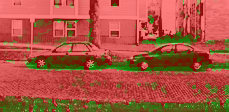}};
			\node at (4.5,-2.2) {\includegraphics[width=4cm]{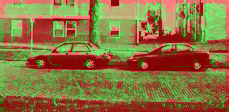}};
			\begin{scope}[xshift=2.105cm,yshift=-1cm,x=0.0421cm,y=0.0421cm]
				\foreach \y / \x / \lbl in {3.45294/-44.7131/1,11.1189/-25.4462/2,7.23007/20.4108/3}
					\node at (\x,-\y) [shape=circle,fill=white,draw=black,anchor=center,inner sep=1pt]{\lbl};
			\end{scope}
			\foreach \x / \y / \lbl in {4.5/0/1,0/-2.2/2,4.5/-2.2/3}
					\node at ($(\x,\y)+(3.7,-0.3)$) [shape=circle,fill=white,draw=black,anchor=center,inner sep=1pt]{\lbl};
		\end{tikzpicture}
		\caption{\textbf{Inhomogeneous appearance model.} For each template pixel a local appearance model was learned. \textit{Top left:} template $X$ with three selected (super-)pixels $\{x_i\}_i$. \textit{Top right, bottom row:} costs $c_\feat(x_i,\cdot)$ for the three selected pixels. %
		The appearance cost of single pixels is not very informative. Only by combining costs from all template pixels and their relative spatial position enables one to find the objects (Fig.~\protect\ref{fig:11}).}
		\label{fig:10}
	\end{figure}
	
	Since the variation of the shapes of the cars is small we only consider translations during branch and bound for locating the cars. Geometric flexibility beyond that is provided by the optimal transport matching. In this way on 10 out of 15 test images the global optimum correctly corresponded to a car (some images show multiple cars). As baseline we performed a simple Hough transform which failed to correctly locate any car. Fig.~\ref{fig:11} gives some example cases and also illustrates a failed case.

	\begin{figure}
		\centering
		\begin{tabular}{cc}
			\includegraphics[height=3cm]{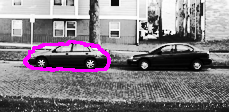} &
			\includegraphics[height=3cm]{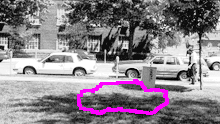} \\
			\includegraphics[height=3cm]{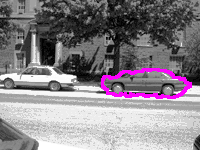} &
			\includegraphics[height=3cm]{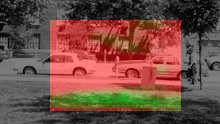} \\			
		\end{tabular}
		\caption{\textbf{Locating cars with a spatially inhomogeneous appearance model.} \textit{Left column.} Two successful examples of locating a car within the test image. \textit{Right column.} Top: A failed example. Bottom: plot of the matching cost depending on translation. Apparently the chosen features are too simple: the shady patch of lawn in the foreground has by far the best cost. Note however that on the right car there is a distinct local minimum.}
		\label{fig:11}
	\end{figure}

	A similar experiment was performed in \cite{LempitskyBranchAndMincut2008}. There the main focus was on modelling the boundary of the cars whereas here we concentrate on its region. Both approaches can incorporate both cues from the object interior as well as its boundary. In Sect.~\ref{sec:OptimizationGraphCut} it was discussed how \cite{LempitskyBranchAndMincut2008} is closely related to the graph-cut relaxation of our functional. The most significant difference is how in our approach the geometric variability is explicitly modelled by a linear space of modes whereas in \cite{LempitskyBranchAndMincut2008} it is implicitly encoded in a hierarchical clustering.
	
	\paragraph{Adaptive $c_\feat$.}
	We have already mentioned in Remark \ref{rem:cFeatTransform} that the Wasserstein modes can also be extended beyond geometric variations to the feature component.
	This is of particular use when an expected feature is known to change under a certain geometric transformation. For example the orientation of an expected gradient changes with rotation. More generally, a vector valued feature $f_x$ will have to be transformed by
	\begin{align}
		D \XEmb_\XCoef(x) = \id + \sum_{i=1}^n \XCoef_i\,D\XMode_i(x)\,,
	\end{align}
	the Jacobian of the applied transformation, to preserve it's `relative orientation' within the template, and we see that this yields a linear deformation on the feature space.
	
	Here we provide a simple example to point out the potential of this flexibility. We now assume that both location and expected feature of a template point vary with the transformations.
	We model this by linearly expanding $c_\feat$ in $\XCoef$ around the origin. That is we choose (c.f.~(\ref{eq:cFeatTransformTransform}-\ref{eq:cFeatTransformCost})):
	\begin{align}
		\hat{c}\big(\XEmbFeat_\XCoef(x),(y,f_y)\big) & {} = \cGeo\big(\XEmb_\XCoef(x),y\big) + c_\feat(f_x,f_y) + \sum_{i=1}^n \XCoef_i \cdot c_{\feat,i}(f_x,f_y)
	\end{align}
	where $c_{\feat,i}(f_x,f_y)$ is the partial derivative of the feature component of $\hat{c}\big(\XEmbFeat_\XCoef(x),(y,f_y)\big)$ w.r.t.~$\XCoef_i$ evaluated at zero (thus giving the first order change along the feature component of $\XModeFeat_i$).
	Both discussed optimization schemes can easily be adapted to this extension.
	
	As a toy example we will be looking for apples. Unripe, small apples are assumed to be green, ripe, large apples should have a reddish color. That is, the expected color varies with size (Naturally the apparent size of an apple on the image depends strongly on the distance from the camera. But we will generously overlook this for the sake of the demonstration.)
	The results of our search for fruit are illustrated in Fig.~\ref{fig:12}.

	\begin{figure}
		\centering
		\includegraphics[width=6cm]{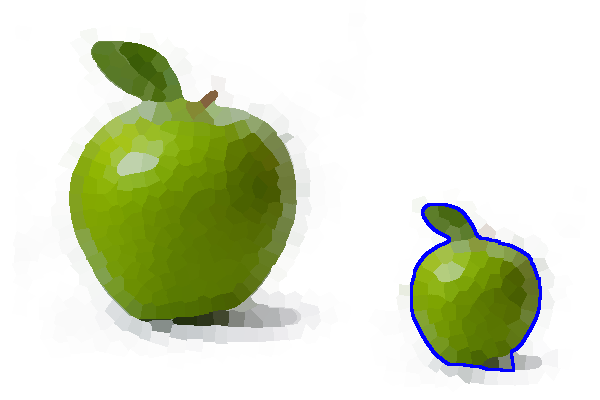}
		\hskip 2cm
		\includegraphics[width=6cm]{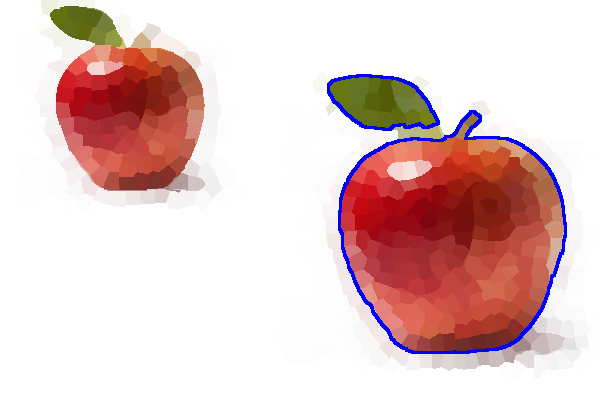}
		\caption{\textbf{Locating apples with dynamic appearance.} We are looking for apples in the test image. According to our model, small apples should appear green (unripe) and large apples reddish. This change in appearance, depending on the geometric state, can be encoded by a Wasserstein mode that extends to the feature cost function. Consequently, the small green and the large red apple are detected, while the `implausible' large green and small red apple are discarded.}
		\label{fig:12}
	\end{figure}
	
	\paragraph{Point Clouds.}
	Last but not least we want to further illustrate the flexibility of the numerical framework by applying it to a scenario with point clouds. This is relevant when one does not deal with dense images but only with sparse interest points. We give a transparent, synthetic example in Fig.~\ref{fig:13}.
	
	\begin{figure}
		\centering
		\includegraphics[height=4cm]{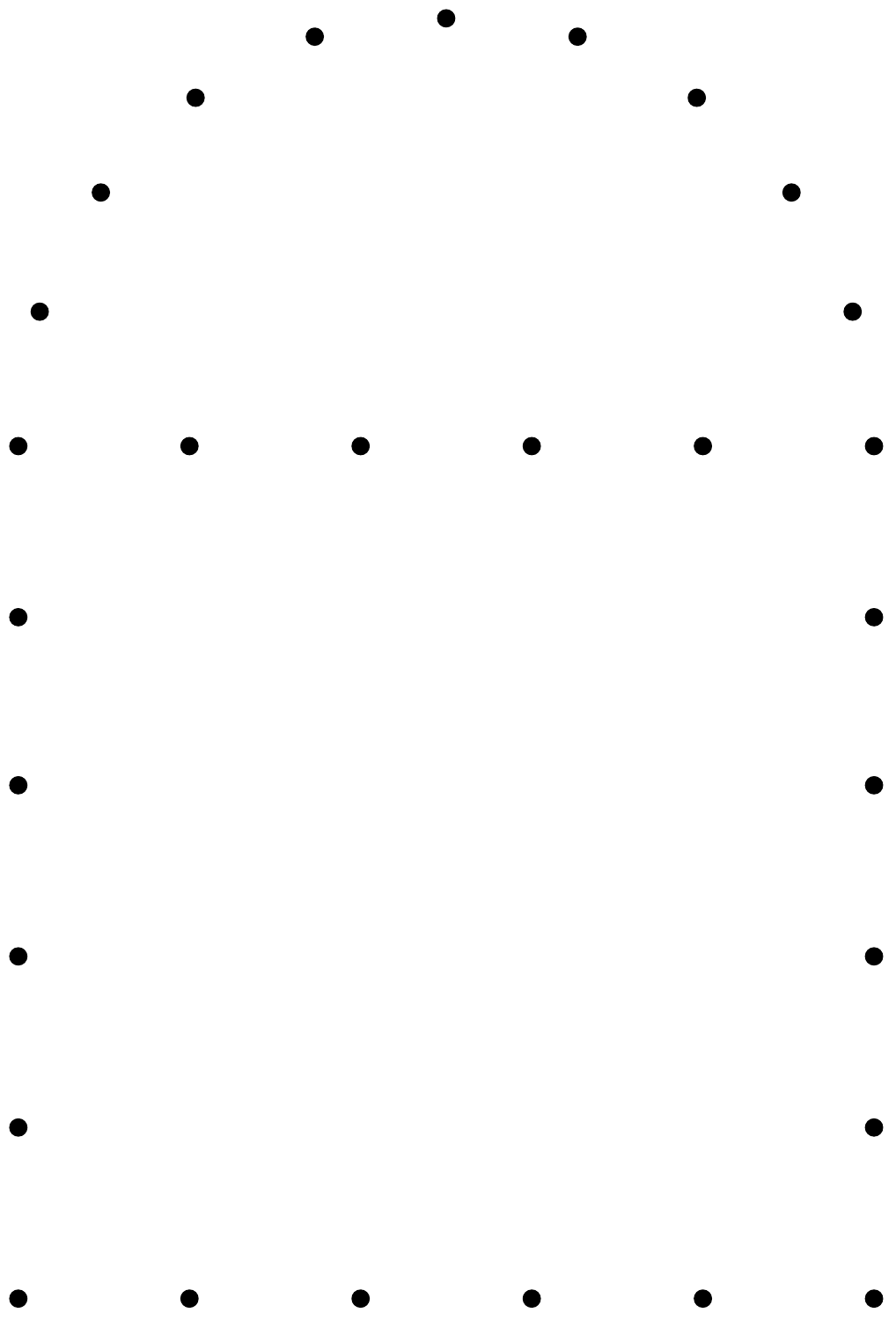}
		\hskip 1cm
		\includegraphics[height=4cm]{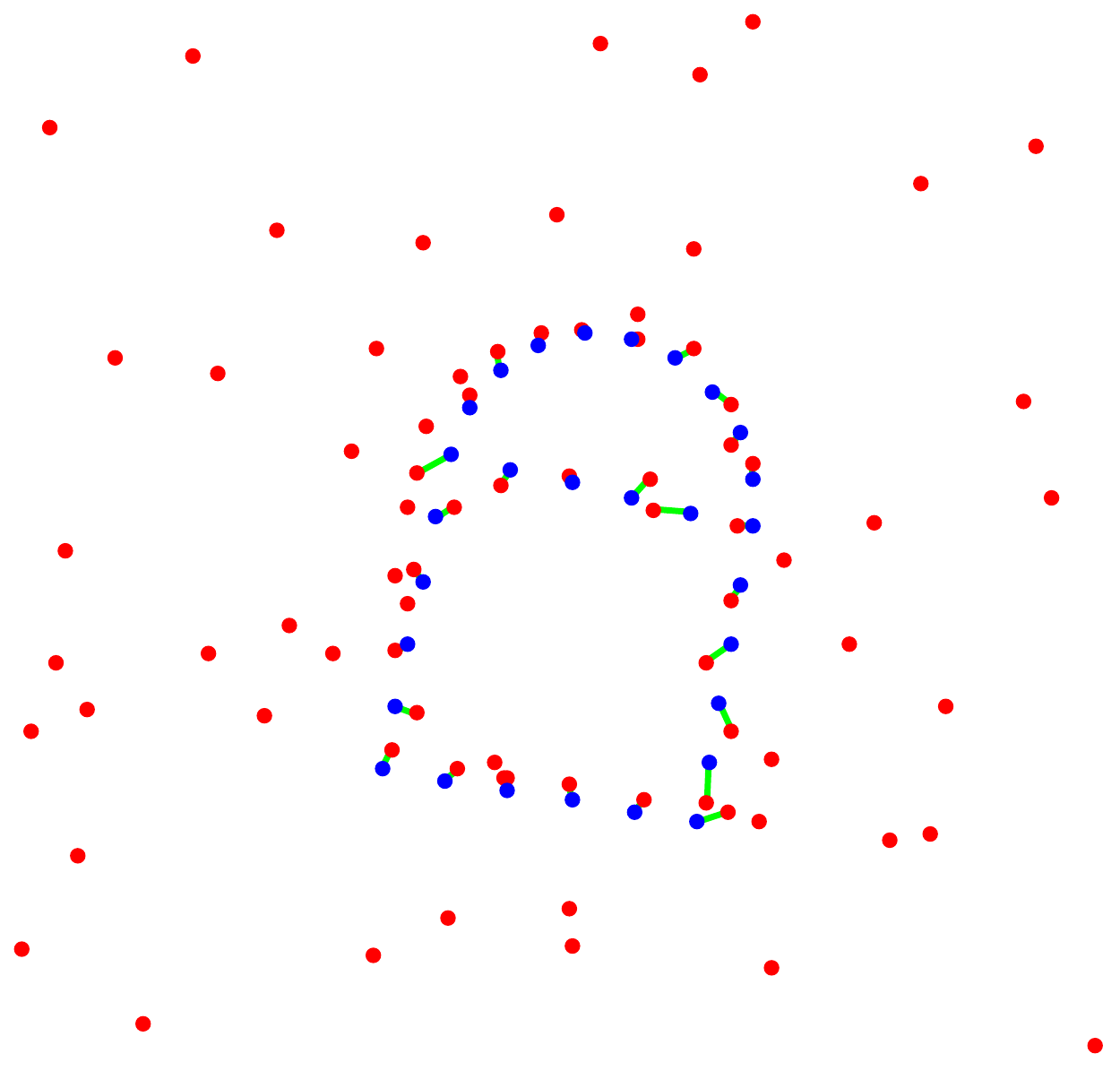}		
		\caption{\textbf{Segmenting and matching on point clouds.} \textit{Left:} the template, a schematic `gate'. \textit{Right:} the original shape is subjected to perspective transformations (foreshortening, rotation, scale), noise and additional, noisy observations are added. With branch and bound the original shape is detected. This could be applied to matching sparse interest points on images.}
		\label{fig:13}
	\end{figure}

\section{Conclusion}
	We have presented a functional for simultaneous image segmentation and shape matching to correctly locate and segment objects within images under noisy conditions.
	
	Matching is based on optimal transport with a cost function that combines geometric plausibility with consistency of appearance features. Through the convex Kantorovich formulation in terms of coupling measures the functional can naturally be combined with other segmentation terms known from convex variational image segmentation.
	To implement geometric invariances and to account for non-isometric shape variations we introduced additional degrees of freedom, drawing from the Riemannian structure of the 2-Wasserstein space. Through an equivalence relation of the class of shape measures with closed contours this enabled us to introduce well established shape analysis tools from the contour regime into the segmentation approach while remaining in the measure representation.
	
	While the resulting functional is non-convex, this non-convexity is constrained to a low dimensional variable which allowed us to devise an adaptive convex relaxation on which a globally optimal branch\,\&\,bound optimization scheme could be constructed.
	Alternatively, a faster but only locally optimal alternating optimization scheme was discussed. While it seems impractical to run the branch and bound scheme on a high number of deformation modes, it still provides a consistent way to find good initializations for the alternating scheme, thus overcoming a severe problem in many other segmentation\,/\,matching approaches. Determining a good initial guess and the subsequent `fine tuning' are based on the very same model and only differ in the application of the optimization scheme.
	To reduce numerical complexity, a graph-cut relaxation was discussed.
	
	In Sect.~\ref{sec:Experiments} we presented a series of numerical examples to demonstrate various aspects of the approach. The basic behaviour of the branch and bound scheme was illustrated as well as the limitations of the alternating scheme. It was shown how the location and shape of the optimal segmentations depend on noise and how different kinds of noise can at least partially be handled by properly choosing the weights between the different terms of the functional.
	We put a particular focus on illustrating the flexibility in both spatial data structure (pixels, super-pixels, point clouds) as well as in incorporating different types of knowledge on the object appearance (spatially varying, adaptive to deformations).
	
	In the presented state a major limitation of the functional is the linearity of the modes: this makes it difficult to handle large deformations. In this respect other approaches such as the LDDMM framework \cite{GlaunesTrouveYounesCVPR2004,LDDMM2005,YounesShape2010} are already much further developed, yet focus on smooth registration mappings without addressing variational segmentation simultaneously and explicitly.
	On the other hand we notice that in terms of handling local feature data this approach is similarly flexible (compare for example with \cite{TrouveFunctionalCurrents2014}).
	Also, we consider the branch and bound scheme as an important step towards coherently solving the initialization problem.	
	
	Future work should therefore focus on making the deformations more flexible and powerful while trying to retain the ability to obtain robust initializations.
	
	{\bf Acknowledgement.}
	This work was supported by the DFG, grant GRK 1653.

%\phantomsection
\addcontentsline{toc}{section}{References}
\bibliography{references-1,references-2}{}
%\bibliography{../../references,../../Schmitzer}{}
\bibliographystyle{plain}

\end{document}